\documentclass{article}

\PassOptionsToPackage{numbers, compress}{natbib}
 \usepackage[preprint]{neurips_2026}

\usepackage[utf8]{inputenc} 
\usepackage[T1]{fontenc}    
\usepackage{hyperref}       
\usepackage{url}            
\usepackage{booktabs}       
\usepackage{amsfonts}       
\usepackage{nicefrac}       
\usepackage{microtype}      
\usepackage{xcolor}         

\usepackage{mathtools}
\usepackage{amssymb}
\usepackage{amsfonts}
\usepackage{amsthm}
\usepackage{thmtools}

\newtheorem{theorem}{Theorem}
\newtheorem{lemma}[theorem]{Lemma}

\newtheorem{assumption}[theorem]{Assumption}
\newtheorem{definition}[theorem]{Definition}

\usepackage{algorithm}
\usepackage{algorithmic}
\usepackage{graphicx}
\usepackage{subcaption}
\usepackage{thm-restate}
\usepackage{multirow}
\usepackage{comment}
\newcommand{\rev}[1]{\textcolor{black}{#1}}
\title{Importance-Guided Basis Selection for \\ Low-Rank Decomposition of Large Language Models}

\author{%
  Daniel Agyei Asante \\
  Iowa State University \\
  \texttt{dasante@iastate.edu} \\
  \And
  Ernie Chang \\
  Meta \\
  \texttt{erniecyc@meta.com} \\
  \And
  Yang Li\thanks{Corresponding author. Address:  2434 Osborn Dr, Ames, IA 50011, United States.} \\
  Iowa State University \\
  \texttt{yangli1@iastate.edu} \\
}

\begin{document}

\maketitle

\begin{abstract}
    
Low-rank decomposition is a compelling approach for compressing large language models, but its effectiveness hinges on selecting which singular-vector bases to retain for a target task. Existing methods such as Basel adapt singular-value coefficients on downstream data and prune bases with small re-learned magnitudes, a heuristic that can be misaligned with task performance because it ignores the local geometry of the loss landscape. We present Basis Selection with Importance (BSI), a principled low-rank compression framework that ranks and prunes bases by directly estimating the expected loss increase incurred when each basis is removed. BSI derives a derivative-based importance score from a second-order Taylor expansion of the task loss with respect to singular values, combining first-order sensitivity and second-order curvature to quantify pruning impact. To make this criterion practical for LLMs, we develop an efficient Hessian-diagonal estimator by adapting the Hutchinson randomized-probing method to loss curvature with symmetric parameter perturbations. We provide a comprehensive theoretical analysis, including loss-increase bounds under basis pruning, explicit propagation of Hessian-diagonal estimation error into these bounds, variance characterization tied to the Hessian spectrum, high-probability sample-complexity guarantees for achieving a target estimation accuracy, and guidance on perturbation intensity. Extensive experiments on mathematical reasoning benchmarks demonstrate that BSI consistently outperforms state-of-the-art low-rank decomposition baselines, with especially strong improvements under deep compression.
\end{abstract}

\section{Introduction}

The remarkable success of large language models (LLMs) has been driven in part by ever-growing parameter counts, but this scaling comes with prohibitive memory footprints and inference latency~\citep{kaplan2020scaling,hoffmann2022training}. Enabling LLM deployment in resource-constrained settings such as mobile and wearable devices therefore motivates an active line of work on model compression. Among existing approaches, low-rank decomposition~\citep{li2025basel,hsu2022fwswd,svd_noach} is particularly appealing due to its strong empirical effectiveness and rigorous mathematical grounding, which leverages the low-dimensional structure often present in neural network weights and activations. By approximating large matrices as the product of smaller low-rank factors, low-rank decomposition methods substantially reduce parameter count and computation while largely preserving model performance.

Standard low-rank compression typically relies on singular value decomposition (SVD)~\citep{SVD_algorithm}. Given a weight matrix $\mathbf{W}\in\mathbb{R}^{n\times m}$, SVD yields the rank-$r$ approximation
\(
\mathbf{W}\approx \mathbf{U}\mathbf{S}\mathbf{V}^\top \;=\; \sum_{i=1}^{r}\sigma_i\,\mathbf{u}_i\mathbf{v}_i^\top,
\)
where $\{\sigma_i\}_{i=1}^r$ are nonnegative singular values and $\{\mathbf{u}_i\}_{i=1}^r\subset\mathbb{R}^n$, $\{\mathbf{v}_i\}_{i=1}^r\subset\mathbb{R}^m$ are orthonormal singular vectors. This representation views $\mathbf{W}$ as a weighted sum of rank-one orthonormal bases $\{\mathbf{u}_i\mathbf{v}_i^\top\}_{i=1}^r$, with weights given by $\sigma_i$. Recent evidence suggests that many such bases encode task-specific semantics (e.g., mathematics, coding, geography) that may be valuable for some downstream applications but unnecessary for others~\citep{li2025basel}. Motivated by this observation, Basel~\citep{li2025basel}, the state-of-the-art low-rank decomposition method for model compression, compresses pretrained models by identifying and removing bases that are irrelevant to the target application. Concretely, Basel augments each SVD factorization with a small set of auxiliary factors:
\begin{equation}
\label{eqn:basel}
\widetilde{\mathbf{W}} \;=\; \sum_{i=1}^{r}\sigma_i\,\mathbf{u}_i\mathbf{v}_i^\top \;+\; \sum_{j=1}^{\tilde r}\tilde{\mathbf{u}}_j\tilde{\mathbf{v}}_j^\top,
\end{equation}
where the rank-one bases $\{\mathbf{u}_i\mathbf{v}_i^\top\}_{i=1}^r$ are fixed from the pretrained SVD, and only their coefficients $\{\sigma_i\}$ are re-learned on target data. The second term introduces $\tilde r$ additional learnable vectors $\{(\tilde{\mathbf{u}}_j,\tilde{\mathbf{v}}_j)\}_{j=1}^{\tilde r}$ to capture domain-specific features absent from pretraining. After optimizing these learnable parameters on the target application, many coefficients $\sigma_i$ shrink toward zero, and Basel prunes the corresponding bases by setting $\sigma_i=0$. This raises a central challenge: deciding which bases are essential for task performance and which can be safely discarded. We refer to this challenge as the \emph{basis selection problem}.

\subsection{The Basis Selection Problem}
\label{subsec:basis-selection}

Basel addresses the basis selection problem using a simple magnitude criterion: after adaptation, it ranks bases by their re-learned singular values and prunes those with small $\sigma_i$, implicitly assuming that a smaller $\sigma_i$ means the corresponding basis \(\mathbf{u}_i\mathbf{v}_i^\top\) is irrelevant to the target application, whereas a larger $\sigma_i$ indicates an essential basis. While intuitive and empirically efficient, this magnitude-based heuristic is not necessarily aligned with the performance of the target task. In particular, treating $\sigma_i$ as a proxy for task relevance overlooks the local geometry of the loss landscape: a basis with a small coefficient can still lie in a direction to which the loss is highly sensitive, so removing it may incur a disproportionate degradation. Moreover, the heuristic ignores curvature (higher-order) information, limiting its ability to reliably predict the impact of pruning.

In this paper, we propose \underline{B}asis \underline{S}election with \underline{I}mportance (BSI), a principled low-rank compression framework that links basis selection directly to the loss landscape. Instead of relying solely on re-learned singular values, BSI evaluates each basis by estimating the expected increase in loss incurred when that basis is removed. Our approach is grounded in a second-order Taylor expansion of the loss $\ell$ with respect to the singular values $\{\sigma_i\}$. This yields a formal importance score for dropping basis $i$:
\(
I_i \triangleq -\sigma_i\,\mathbb{E}\!\left[\frac{\partial \ell}{\partial \sigma_i}\right] \;+\; \frac{1}{2}\sigma_i^2\,\mathbb{E}\!\left[\frac{\partial^2 \ell}{\partial \sigma_i^2}\right].
\)
Unlike magnitude-based heuristics, this criterion explicitly incorporates both first-order sensitivity (gradient) and second-order curvature (Hessian) of the loss, thereby quantifying how pruning a basis affects task performance.  By prioritizing the retention of bases according to their influence on the loss rather than their spectral magnitude, BSI enables more effective compression of large-scale models while preserving performance (Section~\ref{sec:algorithm_design}).


A key technical challenge in realizing this principle for LLMs is to estimate, efficiently and reliably, the second-order term in $I_i$ at scale. Directly computing Hessian information is prohibitive for the re-parameterized model: if the model has \(n\) active singular-value parameters, then the Hessian \(H_{\sigma} = \nabla_{\sigma}^{2}\ell\), whose entries are \(\partial^{2}\ell / (\partial \sigma_i \partial \sigma_j)\), is an \(n \times n\) matrix, requiring \(O(n^2)\) storage and making explicit formation or diagonal extraction infeasible.
To address this, we adapt the Hutchinson estimator~\citep{hutchinson1989stochastic}, originally developed for trace estimation of large implicit matrices, to estimate the Hessian diagonal using randomized symmetric parameter perturbations (Section~\ref{sec:BSI_derivation}).


We provide a comprehensive theoretical analysis of the proposed algorithm. (1) We derive an upper bound on the loss increase induced by basis pruning and quantify how Hessian-diagonal estimation error propagates into this bound (Section~\ref{sec:loss_bound}). (2) We characterize the variance of our Hessian-diagonal estimator and relate it to the Hessian’s spectral structure (Section~\ref{sec:variance_convergence}). (3) We establish sample-complexity guarantees: the number of stochastic probes required to bound the Hessian-diagonal estimation error within $\epsilon$ with probability at least $1-\delta$, explicitly in terms of the Hessian spectrum (Section~\ref{sec:prob_bound}). (4) We analyze the choice of perturbation intensity for accurate Hessian-diagonal estimation (Appendix~\ref{appendix:perturbation-intensity}).

Our contributions are summarized as follows:
\begin{itemize}
    \item We propose \emph{Basis Selection with Importance} (BSI), a principled low-rank decomposition framework for LLM compression. BSI links basis selection to the loss landscape via a derivative-based importance metric that incorporates both gradient and Hessian information, aiming to minimize the expected loss increase incurred by dropping bases.

    \item We provide a rigorous theoretical analysis of BSI, including (i) bounds on the loss impact of basis pruning and its dependence on Hessian-diagonal estimation error, (ii) variance and high-probability guarantees for our Hessian-diagonal estimator, and (iii) an analysis of the perturbation intensity used in the estimator.

    \item We demonstrate the practical effectiveness of BSI through extensive experiments on multiple mathematical reasoning tasks. BSI consistently outperforms state-of-the-art low-rank decomposition baselines, with particularly strong gains under deep compression.
\end{itemize}

\subsection{Related Work}

Model compression addresses the deployment bottlenecks of LLMs including memory, inference latency, and computational cost. To compress over-parameterized networks efficiently, researchers have developed several distinct paradigms, primarily categorized into quantization, pruning, and low-rank decomposition.

Quantization compresses a model by representing weights or activations with fewer bits (e.g., INT8, INT4) rather than FP16/FP32. This reduces memory and can accelerate inference on hardware optimized for arithmetics. A foundational work in this area is LLM.int8()~\citep{dettmers2022llmint8}, which separates the Transformer's feed-forward and attention weight matrices into outlier and non-outlier dimensions based on activation magnitude. Outlier dimensions are computed in higher precision, while the remaining weights are quantized to 8-bit and multiplied using vector-wise scaling derived from row- and column-wise absolute maxima. Similarly, GPTQ~\citep{frantar2023gptq} performs block-wise quantization of weight matrices using a second-order approximation of the loss to minimize quantization error.


Pruning reduces model size by identifying and removing redundant parameters from a neural network. Early work by \citet{han2015deep} introduced iterative magnitude pruning, which learns the importance of individual connections and removes those considered unimportant. More recently, \citet{wanda} proposed Wanda, which estimates weight importance using both weight magnitude and the norm of the corresponding input activation. To improve inference efficiency, \citet{flap} introduced FLAP, a structured pruning method based on activation fluctuation. FLAP measures how removing individual weight columns changes the output feature map, then prunes components that cause only small fluctuations, helping preserve model performance.


Low-rank compression approximates a large weight matrix as the product of lower-dimensional matrices, thereby reducing the number of parameters as well as memory and computational costs. A standard approach in this direction is SVD~\citep{SVD_algorithm, svd_noach}. To improve SVD, \citet{hsu2022fwswd} propose a weighted low-rank factorization (FWSVD) that incorporates gradient information via Fisher-style weighting to preserve task performance. Recent work by \citet{li2025basel} interprets singular vectors as semantic bases that encode domain-specific information, and performs basis selection to preserve task performance using the magnitude of the singular values as an importance score.

Despite their empirical success, existing low-rank decomposition methods either minimize reconstruction error or adopt heuristic basis selection rules, objectives that can be poorly aligned with the quantities that ultimately matter—downstream loss and accuracy. We investigate this discrepancy through rigorous theoretical analysis, and address it by introducing a principled, loss-grounded importance measure to select bases during compression.

\section{Preliminaries}
\paragraph{Notation. \label{notation}}
To avoid ambiguity, we distinguish the two diagonal operators used throughout the paper.
The operator $\mathrm{Diag}\!:\mathbb{R}^n \to \mathbb{R}^{n\times n}$ maps a vector
$\mathbf{a}\in\mathbb{R}^n$ to the diagonal matrix whose $i$-th diagonal entry equals $a_i$
(with all off-diagonal entries equal to zero). Conversely,
$\mathrm{diag}\!:\mathbb{R}^{n\times n}\to\mathbb{R}^n$ extracts the diagonal of a matrix, i.e.,
\(
\mathrm{diag}(\mathbf{A}) \triangleq [A_{11}, A_{22}, \ldots, A_{nn}]^\top,
\)
where $\mathbf{A}\in\mathbb{R}^{n\times n}$ and $A_{ii}$ denotes the $(i,i)$-th entry of $\mathbf{A}$.
We use this convention consistently throughout the paper.

\subsection{Hutchinson Trace and Diagonal Estimators}

 \begin{definition}[Rademacher Vectors]
\label{def:rademacher}
A random vector $\mathbf{z} = [z_1, z_2, \dots, z_n]^\top$ is a Rademacher vector if its entries $z_i$ are independent and identically distributed (i.i.d) random variables taking values in $\{-1,+1\}$, with probability
\(
\Pr(z_i = \pm 1) = \tfrac{1}{2}.
\)

\end{definition}

Let $\mathbf{z} \in \{-1,+1\}^n$  be a Rademacher random vector. Hutchinson's trace estimator $ T^{(s)}$ for a matrix $\mathbf{A} \in \mathbb{R}^{n \times n}$ is defined as:
\(
T^{(s)}(\mathbf{A}) = \frac{1}{s}\sum_{k=1}^{s} (\mathbf{z}^{(k)})^\top \mathbf{A}\,\mathbf{z}^{(k)} .
\)
It is well-established that when \(\mathbf{A}\) is symmetric, $\mathbb{E}[T^{(s)}(\mathbf{A})]=\mathrm{tr}(\mathbf{A})$ and
\(
\mathrm{Var}[T^{(s)}(\mathbf{A})]=\frac{2}{s}\,\|\bar{\mathbf{A}}\|_F^2,
\)
where $\bar{\mathbf{A}}=\mathbf{A}-\mathrm{Diag}(\mathrm{diag}(\mathbf{A}))$ denotes the matrix obtained from $\mathbf{A}$ by setting its diagonal entries to zero.

The Hutchinson diagonal estimator can be understood as a localized extension of the trace estimator. While the trace estimator $T^{(s)}(\mathbf{A})$ for a matrix $\mathbf{A}$ collapses the quadratic form $\mathbf{z}^\top \mathbf{A}\mathbf{z}$ into a single scalar, the diagonal estimator preserves the entry-wise structure of the operator.

Following the formulation of \citep{bekas2007estimator}, the diagonal of a matrix $\mathbf{A} \in \mathbb{R}^{n \times n}$ can be estimated by considering the Hadamard product of a probe vector $\mathbf{z}$ and its transformation $\mathbf{Az}$. Formally, given a sequence of $s$ number of i.i.d.~Rademacher probe vectors $\{\mathbf{z}^{(k)}\}_{k=1}^s$, Hutchinson's diagonal estimator ${\mathrm{D}}^{(s)}$ is:
\begin{equation}
    \mathrm{D}^{(s)}(\mathbf{A})
 = \frac{1}{s} \sum_{k=1}^s \mathbf{z}^{(k)} \odot (\mathbf{Az}^{(k)}),
\end{equation}
where $\odot$ denotes the element-wise (Hadamard) product.



\subsection{Auxiliary Definitions}

\begin{definition}[Riemann Zeta Function]
\label{def:zeta}
The Riemann zeta function is defined for real $s > 1$ as
\[
\zeta(s) \coloneqq \sum_{k=1}^{\infty} \frac{1}{k^{s}}.
\]
\end{definition}
It is classical that the series converges for all $s>1$.

\begin{definition}[Generalized harmonic number / partial zeta sum]
\label{def:harmonic}
For $n\in\mathbb{N}$ and $s\in\mathbb{R}$, the \emph{generalized harmonic number}
is defined by
\[
H_{n,s}
\;:=\;
\sum_{k=1}^{n} \frac{1}{k^{s}}.
\]
\end{definition}
$H_{n,s}$ increases monotonically in $n$. $H_{n,s}$ converges to $\zeta(s)$ as $n\to\infty$ for every $s>1$.



\begin{lemma}[Entrywise Spectral Decomposition]
\label{lem:entrywise-spectral}
Let $\mathbf{H} \in \mathbb{R}^{n \times n}$ be a symmetric matrix. Then there exists a set of
orthonormal eigenvectors $\{\mathbf{u}_k\}_{k=1}^n$ with corresponding eigenvalues
$\{\lambda_k\}_{k=1}^n$ such that
\(
\mathbf{H} = \sum_{k=1}^n \lambda_k\mathbf{u}_k \mathbf{u}_k^\top .
\)
In particular, the entries of $\mathbf{H}$ satisfy
\(
H_{ij} = \sum_{k=1}^n \lambda_k u_{ki} u_{kj} 
\) where $u_{ki}$ is the $i$-th element of the $k$-th eigenvector.
\end{lemma}




\section{Basis Selection with Importance (BSI) \label{sec:method}}

Our BSI framework establishes a principled selection rule that minimizes the expected loss increase incurred by dropping bases. It links the \emph{basis selection} problem to the local geometry of the loss landscape by deriving an importance metric from the first- and second-order derivatives of the loss $\mathcal{\ell}$ with respect to the singular values. We first describe the framework design in Section~\ref{sec:algorithm_design}, and then provide a detailed theoretical analysis in Section~\ref{sec:theoretical_analysis}. Throughout, we adopt the denominator-layout convention for matrix calculus.


\subsection{Design of the BSI Algorithm \label{sec:algorithm_design}}

\subsubsection{Importance Metric for Selecting Bases \label{sec:BSI_derivation}}
Let $\mathbf{W}$ denote the weight matrix of a linear layer. Using the singular value decomposition, we write
\(
\mathbf{W} =  \sum_{i} \sigma_i \mathbf{u}_i \mathbf{v}_i^\top.
\)
Given an input $\mathbf{x}$, the layer output is
\(
\mathbf{y}=\mathbf{W}\mathbf{x}=\sum_{i}\sigma_i\,\mathbf{u}_i \mathbf{v}_i^\top \mathbf{x}.
\)
Dropping the $i$-th basis corresponds to removing its contribution, i.e., setting $\sigma_i=0$. Therefore, the induced change in the singular value is
\(
\Delta\sigma_i = 0-\sigma_i = -\sigma_i,
\)
which in turn yields the following expected change in the loss:
\begin{align*}
\mathbb{E}\!\left[\Delta \ell\right] & \equiv \mathbb{E}\!\left[\ell(\sigma_i + \Delta \sigma_i) - \ell(\sigma_i)\right] \\
&\approx
\mathbb{E}\!\left[\left(
\ell(\sigma_i)
+
\frac{\partial \ell}{\partial \sigma_i}\Delta \sigma_i
+
\frac{1}{2}\frac{\partial^2 \ell}{\partial \sigma_i^2}(\Delta \sigma_i)^2
\right)
-
\ell(\sigma_i)\right] \\
&=
\mathbb{E}\!\left[\frac{\partial \ell}{\partial \sigma_i}\Delta \sigma_i
+ \frac{1}{2}\frac{\partial^2 \ell}{\partial \sigma_i^2}(\Delta \sigma_i)^2\right] \\
&= -\sigma_i\, \mathbb{E}\!\left[\frac{\partial \ell}{\partial \sigma_i}\right]
\;+\;
\frac{1}{2}\sigma_i^2\, \mathbb{E}\!\left[\frac{\partial^2 \ell}{\partial \sigma_i^2}\right].
\end{align*}
Considering this, we define importance metric for basis $i$ as
\[
I_i \;\triangleq\;
-\sigma_i\, \mathbb{E}\!\left[\frac{\partial \ell}{\partial \sigma_i}\right]
\;+\;
\frac{1}{2}\sigma_i^2\, \mathbb{E}\!\left[\frac{\partial^2 \ell}{\partial \sigma_i^2}\right],
\]
where $\mathbb{E}\!\left[\frac{\partial \ell}{\partial \sigma_i}\right]$ denotes the gradient with respect to $\sigma_i$, and
$\mathbb{E}\!\left[\frac{\partial^2 \ell}{\partial \sigma_i^2}\right]$ denotes the second derivative with respect to $\sigma_i$. Pruning bases with a larger importance metric is expected to induce a larger increase in the loss; therefore, we should avoid pruning such bases whenever possible. Computing this importance metric requires estimating the second derivative with respect to each singular value, which can be obtained via Theorem~\ref{thm:hessian_diag_est}.

\begin{restatable}[Hessian Diagonal Estimator]{theorem}{HessDiagEstimate}
\label{thm:hessian_diag_est}
For any parameter $\sigma_i$ of the model,
\[
\mathbb{E}_{d\sim \mathcal{D}}\!\left[\left.\frac{\partial^2 \ell}{\partial \sigma_i^2}\right|_{\boldsymbol{\sigma} = \boldsymbol{\sigma}^\ast}\right] +\; O(\epsilon^2)
\;=\;
\mathbb{E}_{\substack{\mathbf{z} \sim \mathcal{R}(\mathcal{Z}) \\ d \sim \mathcal{D}}}
\left[
\left(
\left.\frac{\partial \ell(d)}{\partial \boldsymbol{\sigma}}\right|_{\boldsymbol{\sigma} = \boldsymbol{\sigma}^\ast + \frac{\epsilon}{2} \mathbf{z}}
-
\left.\frac{\partial \ell(d)}{\partial \boldsymbol{\sigma}}\right|_{\boldsymbol{\sigma} = \boldsymbol{\sigma}^\ast - \frac{\epsilon}{2} \mathbf{z}}
\right)_i
\cdot
\frac{z_i}{\epsilon}
\right],
\]
where $0 < \epsilon < 1$ denotes the perturbation intensity, and $\mathbf{z}=[z_1,z_2,\dots,z_n]^\top$ is a Rademacher random vector defined in Definition~\ref{def:rademacher}.
\end{restatable}

\begin{proof}
    The proof can be found in Appendix~\ref{apdx:hessian_diagona_estimate_proof}
\end{proof}

Theorem~\ref{thm:hessian_diag_est} shows that the Hessian diagonal estimate can be obtained using only gradients, by applying symmetric perturbations along randomly sampled Rademacher vectors.

\subsubsection{Algorithmic Design \label{subsection:BSI_alg_formulation}}

The BSI procedure is shown in Algorithm \ref{alg:bsi}. The algorithm compresses a pretrained model $\mathcal{M}$ to a target size defined by the \texttt{KeepRatio} through iterative pruning. The process begins by converting linear layer weight matrices, excluding the embedding layer, into the basis-form defined in Eq. \eqref{eqn:basel}. Following \citet{li2025basel}, the model is fine-tuned both before and after the pruning stage to improve adaptation to the target application.

Prior to pruning, the first- and second-order components of the importance metrics $I_i$ are computed. The gradient (the first-order term) is obtained via back-propagation pass. Theorem~\ref{thm:hessian_diag_est} provides a principled way to estimate the second-order term (Hessian) for large models. To achieve this efficiently, a candidate pool $\mathcal{C}$ is first identified (see Appendix~\ref{sec:candidate_pool}) to reduce computational cost.  Bases in this pool are considered for the second-order computation. The model is perturbed along both positive and negative directions using $\mathbf{z}$ 
and $\epsilon$, where $\mathbf{z}$ is a Rademacher random vector and $\epsilon \in \mathbb{R}$ denotes the perturbation intensity.
The gradients $\mathbf{g}^{+}$ and $\mathbf{g}^{-}$ are computed at these perturbed states 
 yielding the Hessian diagonal estimate
\(
\mathbb{E}\!\left[\frac{\partial^2 \ell}{\partial \sigma_i^2}\right] = \big(\mathbf{g}^{+} - \mathbf{g}^{-}\big)_i \cdot \frac{z_i}{\epsilon}.
\)
The final importance score for dropping the $i$-th basis combines the gradient-based first-order term with the Hessian-diagonal-based second-order term. Both quantities are averaged over multiple mini-batches for stability. At each pruning step, we remove bases with the smallest $I_i$ until the remaining set satisfies the \texttt{KeepRatioPerPruning} threshold.

\begin{algorithm}[ht]
\small
\caption{BSI Algorithm}
\label{alg:bsi}
\begin{algorithmic}[1]
\STATE \textbf{Input:} Pretrained or fine-tuned model $\mathcal{M}$
\STATE \textbf{Output:} Compressed model $\widetilde{\mathcal{M}}$
\STATE \textbf{Data:} \texttt{PruningEpochs}; \texttt{PruningRounds}; \texttt{NumIterPerEpoch}; \texttt{SamplingIterRatio}; \texttt{KeepRatio}; Perturbation Intensity $\epsilon$; Dataset $ \mathcal{D}$

\STATE $\texttt{IterPerPruning} \gets 
\mathrm{round}\!\left(
\text{ \texttt{NumIterPerEpoch} $\times $ \texttt{PruningEpochs} }\text{ / }{\texttt{PruningRounds}} 
\right)$
\STATE $\texttt{KeepRatioPerPruning} \gets (\text{\texttt{KeepRatio}})^{1/\texttt{PruningRounds}}$
\STATE $\texttt{NumProfilingIter} \gets \mathrm{round}\!\left(\texttt{SamplingIterRatio} \times \texttt{IterPerPruning}\right)$

\FORALL{$\text{linear layer }\ell \in \mathcal{M}$ --- excluding the embedding layer}
    \STATE Convert weight matrix $\mathbf{W}^{(\ell)}$ into the form of Eq. \eqref{eqn:basel}, yielding active bases $\mathcal{A}^{(\ell)}$
\ENDFOR

  \STATE $\mathbb{E}\!\left[\frac{\partial \ell(d)}{\partial \sigma_i}\right] \gets 0,\quad
       \mathbb{E}\!\left[\frac{\partial^2 \ell(d)}{\partial \sigma_i^2}\right] \gets 0$
\FOR{$\texttt{iter}=1$ \TO 
$\text{\texttt{NumIterPerEpoch} $\times $ \texttt{PruningEpochs}}$}
    \STATE Sample $d \sim \mathcal{D}$
    \STATE $\texttt{IterInBlock} \gets ((\texttt{iter}-1) \bmod \texttt{IterPerPruning}) + 1$
    
    \IF{$\texttt{IterInBlock} \le \texttt{IterPerPruning} - \texttt{NumProfilingIter}$}
        \STATE Tune learnable parameters in Eq. \eqref{eqn:basel} for all layers using $d$
    \ELSE
        \STATE Compute gradients $\frac{\partial \ell(d)}{\partial { \boldsymbol{\sigma}}}$ for all layers in the model
        \STATE $\mathbb{E}\!\left[\frac{\partial \ell(d)}{\partial \boldsymbol{\sigma}}\right]
        \gets
        \mathbb{E}\!\left[\frac{\partial \ell(d)}{\partial \boldsymbol{\sigma}}\right]
        +
        \frac{1}{\texttt{NumProfilingIter}}
        \frac{\partial \ell(d)}{\partial \boldsymbol{\sigma}}$

        \FORALL{$\text{linear layer }\ell \in \mathcal{M} $ --- excluding the embedding layer}
            \STATE Let $\mathcal{A}\gets \mathcal{A}^{(\ell)}$ with singular values $\{\sigma_i\}_{i\in\mathcal{A}}$
            \STATE Define candidate pool $\mathcal{C}\subseteq\mathcal{A}$ based on Eq. \eqref{eqn:keeping_set}
            \STATE Sample $\mathbf{z} \sim \mathcal{R}(\mathcal{Z})$, where $z_i = 0$ for all $i \notin \mathcal{C}$
            
            \STATE Compute $\mathbf{g}^{+} \gets \left.\frac{\partial \ell(d)}{\partial {\boldsymbol{\sigma}}}\right|_{{\boldsymbol{\sigma}} = {\boldsymbol{\sigma}}^\ast + \dfrac{\epsilon}{2} \mathbf{z}}$
            \STATE Compute $\mathbf{g}^{-} \gets \left.\frac{\partial \ell(d)}{\partial {\boldsymbol{\sigma}}}\right|_{{\boldsymbol{\sigma}} = {\boldsymbol{\sigma}}^\ast - \dfrac{\epsilon}{2} \mathbf{z}}$
            
            \FORALL{$i \in \mathcal{C}$}
                \STATE $\mathbb{E}\!\left[\frac{\partial^2 \ell(d)}{\partial \sigma_i^2}\right] \gets 
                \mathbb{E}\!\left[\frac{\partial^2 \ell(d)}{\partial \sigma_i^2}\right]
                +
                \frac{1}{\texttt{NumProfilingIter}}
                \big(\mathbf{g}^{+}-\mathbf{g}^{-}\big)_i \cdot \frac{z_i}{\epsilon}$
            \ENDFOR
        \ENDFOR
    \ENDIF

    \IF{$\texttt{IterInBlock} = \texttt{IterPerPruning}$}
        \FORALL{$\text{linear layer }\ell \in \mathcal{M}$ --- excluding the embedding layer}
            \STATE $I_i \gets -\sigma_i \,\mathbb{E}\!\left[\frac{\partial \ell(d)}{\partial \sigma_i}\right]
            + \frac{1}{2}\sigma_i^2\,\mathbb{E}\!\left[\frac{\partial^2 \ell(d)}{\partial \sigma_i^2}\right]$
          \STATE $\mathbb{E}\!\left[\frac{\partial \ell(d)}{\partial \sigma_i}\right] \gets 0,\quad
       \mathbb{E}\!\left[\frac{\partial^2 \ell(d)}{\partial \sigma_i^2}\right] \gets 0$

            \STATE \parbox[t]{0.9\linewidth}{
            Prune bases with the smallest importance scores $I_i$ from the candidate pool $C$,
subject to the constraint that the total importance score of the remaining bases in $C$
is equal to $\texttt{KeepRatioPerPruning}$ times the total positive importance score in $C$
before pruning (if all bases in the candidate pool have negative importance scores, we prune the entire pool).
}
        \ENDFOR
       
    \ENDIF

\ENDFOR

\STATE \textbf{return} $\widetilde{\mathcal{M}}$
\end{algorithmic}
\end{algorithm}

\subsection{Theoretical Analysis of BSI \label{sec:theoretical_analysis}}

We first bound the loss change induced by dropping a basis, considering both the exact Hessian setting and the case with estimator error (Section~\ref{sec:loss_bound}). We then analyze the variance of the Hutchinson diagonal estimator and derive a probability bound for the estimator error leveraging the power-law decay of neural network Hessian spectrum (Sections~\ref{sec:variance_convergence} and~\ref{sec:prob_bound}). In addition,   we analyze the perturbation intensity of the Hessian estimate (Appendix~\ref{appendix:perturbation-intensity}). Finally, we prove that the Hessian diagonal estimate for the output layer is non-negative (Appendix~\ref{appendix:proof-prop-lm-head}).

\subsubsection{Loss Change Under Basis Pruning 
\label{sec:loss_bound}}
In Theorem~\ref{thm:basis_pruning_bound_on_loss}, we give a bound on the loss induced by pruning a basis in the idealized setting where the Hessian diagonal is exact. 
Theorem~\ref{thm:basis_pruning_relative_hessian_error} extends the bound to the practical setting where the Hessian diagonal are estimated with error.

\begin{restatable}[Bound on Loss Change from Pruning Bases]{theorem}{LossBoundGeneral}
\label{thm:basis_pruning_bound_on_loss}
Let $ \ell:\mathbb{R}^r\to\mathbb{R}$ be the loss as a function of the basis singular values $\mathbf{s}$, and assume the Hessian $\nabla^2  \ell$ is $\rho$-Lipschitz:
\[
\|\nabla^2 \ell(\mathbf{x}) - \nabla^2  \ell(\mathbf{y})\|_2
\le \rho \|\mathbf{x} - \mathbf{y}\|_2
\quad \text{for all } \mathbf{x},\mathbf{y}.
\]
Pruning basis $i$ leads to a change in loss following the bound
\[
|\Delta  \ell|
=
\big| \ell(\mathbf{s}-s_i\mathbf{e}_i)- \ell(\mathbf{s})\big|
\le
\left|-g_i s_i + \frac{1}{2}H_{ii}s_i^2\right|
+
\frac{\rho}{6}|s_i|^3,
\]
where $g_i := \nabla  \ell(\mathbf{s})^\top \mathbf{e}_i$, $H_{ii}:=\mathbf{e}_i^\top \nabla^2  \ell(\mathbf{s})\,\mathbf{e}_i$,
and $\mathbf{e}_i=[0,\ldots,0,1,0,\ldots,0]^\top$ with a $1$ on its $i$-th dimension.
\end{restatable}

\begin{proof}
    We provide the proof in Appendix~\ref{appendix:proof-thm-loss-bound-on-basis-pruning}
\end{proof}

\begin{restatable}[Bound on Loss Change for Basis Pruning with Relative Hessian Error]{theorem}{LossBoundRelativeError}
\label{thm:basis_pruning_relative_hessian_error}
Let $ \ell:\mathbb{R}^r \to \mathbb{R}$ be the loss as a function of the basis singular
values $\mathbf{s}$ and assume the Hessian of $ \ell$ is $\rho$-Lipschitz,
Hessian estimate $\widehat H_{ii}$ has a maximum relative error of $\varepsilon$:
\[
\lvert \widehat H_{ii} - H_{ii} \rvert \le \varepsilon \lvert H_{ii} \rvert ,
\]
where \(0 < \varepsilon < 1\) and $H_{ii}$ is the $i$-th diagonal element of the Hessian.
Pruning basis $i$ leads to a change of the loss satisfying
\begin{align*}
\lvert \Delta  \ell \rvert
&=
\lvert  \ell(\mathbf{s} - s_i \mathbf{e}_i) -  \ell(\mathbf{s}) \rvert \\
&\le
\left|
- g_i s_i
+ \frac{1}{2} \widehat H_{ii} s_i^2
\right|
+
\frac{\varepsilon}{2(1-\varepsilon)}
\lvert \widehat H_{ii} \rvert s_i^2
+
\frac{\rho}{6}\lvert s_i \rvert^3 ,
\end{align*}
where $g_i$ is the gradient with respect to $s_i$ and
$\mathbf{e}_i = [0,\dots,0,1,0,\dots,0]^\top$ with $1$ on its i-th dimension
\end{restatable}

\begin{proof}
    See Appendix~\ref{appendix:proof-them-bound-on-loss-with-rel-error} for the proof. 
\end{proof}


\subsubsection{Convergence of Hessian Diagonal Estimator \label{sec:variance_convergence}}

The neural network Hessian has a structural property that does not generally hold for an arbitrary matrix. This property appears in the decay of its eigenvalues. As shown in prior work \citep{Sagun2018Hessian, tang2025investigating}, the eigenvalues of neural network Hessians exhibit power-law decay. We state this behavior as Assumption~\ref{assump:spectral-power-law}. In Appendix~\ref{sec:sv_hessian_spectrum}, we provide empirical evidence that this decay also appears in the Hessian of the reparameterized model in the singular-value space. Under this assumption, Theorem~\ref{thm:hutchinson-variance-dimensionality} provides a bound on the total variance of the diagonal estimator for neural network Hessians.



\begin{assumption}[Spectral Power-Law Decay]
\label{assump:spectral-power-law}
Let $\mathbf{H} \in \mathbb{R}^{n \times n}$ be a symmetric matrix with eigenvalues
$|\lambda_1| \ge |\lambda_2| \ge \dots \ge |\lambda_n|$.
We assume there exist $\alpha > \tfrac{1}{2}$ such that
\[
|\lambda_k| \le |\lambda_1| k^{-\alpha}, \qquad \text{for all } k \ge 1.
\]
\end{assumption} 
\begin{restatable}[Variance of Hessian Estimator]{theorem}{VarDimensionalityHessian}
\label{thm:hutchinson-variance-dimensionality}
Let $\mathbf{H} \in \mathbb{R}^{n \times n}$ be the Hessian of a neural network,
and let ${\mathrm{D}}^{(s)}(\mathbf{H}) = \frac{1}{s} \sum_{k=1}^s \mathbf{z}^{(k)} \odot (\mathbf{H}\mathbf{z}^{(k)})$ be the diagonal estimator  computed using $s$ number of independent Rademacher vectors
$\mathbf{z}^{(1)},\dots,\mathbf{z}^{(s)} \in \{-1,1\}^n$.
Under spectral power-law decay (Assumption~\ref{assump:spectral-power-law}), we have
\(
\sum_{i=1}^n \mathrm{Var}\!\big(\mathrm{D}^{(s)}(\mathbf{H})_i\big)
\le \frac{1}{s}|\lambda_1|^2 \zeta(2\alpha).
\)
This upper bound is independent of the dimension $n$.
\end{restatable}

\begin{proof}
    We present this proof in Appendix~\ref{apdx:variance-dimensionality-convergence-proof}
\end{proof}
For a general fixed matrix $\mathbf{A} \in \mathbb{R}^{n \times n}$, the Hutchinson diagonal estimator is unbiased, and its mean-square error satisfies
\(
\mathbb{E}\!\left[\|{\mathrm{D}}^{(s)}(\mathbf{A}) - \operatorname{diag}(\mathbf{A})\|_2^2\right]
= \sum_{i=1}^n \operatorname{Var}\!\big({\mathrm{D}}^{(s)}(\mathbf{A})_i\big)
\le \frac{1}{s}\|\mathbf{A}\|_F^2,
\)
which in general grows with the ambient dimension $n$. Consequently, worst-case guarantees for diagonal estimation typically deteriorate as the matrix dimension increases, unless additional structure information is imposed. However, when $\mathbf{A}$ arises as a neural network Hessian and its spectrum exhibits power-law decay, Theorem~\ref{thm:hutchinson-variance-dimensionality} shows that the bound on the total variance of the diagonal estimator is controlled by the spectral decay rather than by the dimension.

\subsubsection{A Probabilistic Guarantee Of Hessian Diagonal Estimator\label{sec:prob_bound}}

Variance analyses such as those in Section~\ref{sec:variance_convergence} characterize convergence in expectation, but can be sensitive to large deviations (outliers) in the stochastic estimates. In this section, we derive an $(\varepsilon,\delta)$ probabilistic guarantee on the number of samples needed to estimate the Hessian diagonal. Definition~\ref{def:eps-delta-approximator} formalizes this notion by requiring that the relative error exceeds $\varepsilon$ with probability at most $\delta$. 
Theorem~\ref{theorem:prob-hessian-diagonal} bounds the number of samples $s$ needed for $\mathrm{D}^{(s)}$ to be a relative $(\varepsilon,\delta)$-approximator in terms of Hessian spectrum structure. Corollary~\ref{corollary:prob-bound-PSD} presents a bound on the minimum number of samples necessary to reliably estimate the diagonal of a positive semidefinite Hessian.

\begin{definition}[Relative $(\varepsilon,\delta)$-approximator]
\label{def:eps-delta-approximator}
For any $\varepsilon > 0$ and $0 < \delta < 1$,
the diagonal estimator $\mathrm{D}^{(s)}(\mathbf{A})$ is a relative $(\varepsilon,\delta)$-approximator if 
\(
\text{ }
\Pr\!\left(
\left\| \mathrm{D}^{(s)}(\mathbf{A}) - \operatorname{diag}(\mathbf{A}) \right\|_2
\le
\varepsilon \left\| \operatorname{diag}(\mathbf{A}) \right\|_2
\right)
\ge 1 - \delta,
\)
where $s$ is the number of samples and $\operatorname{diag}(\mathbf{A})$ denotes the diagonal of the matrix $\mathbf{A}$.
\end{definition}

\begin{restatable}[Probabilistic bound for Hessian Diagonal]{theorem}{ProbHessianDiagonal}
\label{theorem:prob-hessian-diagonal}
Assume $\mathbf{H}\in\mathbb{R}^{n\times n}$ is symmetric, $\operatorname{Tr}(\mathbf{H})\neq 0$, and the power-law decay condition (Assumption~\ref{assump:spectral-power-law}) holds,
then the Rademacher diagonal estimator $\mathrm{D}^{(s)}(\mathbf{H})$ is a relative
$(\varepsilon,\delta)$-approximator if
\[
s
>
2\left(
\frac{|\lambda_1^2| H_{n,2\alpha}}{\tfrac{1}{n}\operatorname{Tr}(\mathbf{H})^2}
- 1
\right)
\left(
\frac{\ln(2n/\delta)}{\varepsilon^2}
\right)
=
2\left(
\frac{|\lambda_1^2| H_{n,2\alpha}}{\tfrac{1}{n}\left(\sum_{i=1}^n \lambda_i\right)^2}
- 1
\right)
\left(
\frac{\ln(2n/\delta)}{\varepsilon^2}
\right),
\
\]
where $H_{n,2\alpha}$ is the generalized harmonic number defined in Definition~\ref{def:harmonic}.
\end{restatable}

\begin{proof}
    We show the detailed proof in Appendix~\ref{appendix:prob-bound-section}.
\end{proof}

\begin{restatable}[Positive Semidefinite Hessian Diagonal Estimator Probabilistic Bound]{corollary}{ProbBoundPSDHessian}
\label{corollary:prob-bound-PSD}
If the power law decay condition (Assumption~\ref{assump:spectral-power-law}) holds and that $\mathbf{H}$ is positive semidefinite
with $\lambda_1 \neq 0$,
then the Rademacher diagonal estimator $\mathrm{D}^{(s)}(\mathbf{H})$ is a relative
$(\varepsilon,\delta)$-approximator if
\[
s
>
2\left(
n\,\zeta(2\alpha) - 1
\right)
\left(
\frac{\ln(2n/\delta)}{\varepsilon^2}
\right).
\]
\end{restatable}

\begin{proof}
   We prove this in Appendix~\ref{appendix:prob-bound-subsection}.
\end{proof}

Our probabilistic bound goes beyond generic worst-case matrix bounds, which are often impractical because they depend on information about the implicit matrix that is typically unavailable. In the specific case of neural network Hessians, eigenvalues are widely observed to exhibit power-law decay \citep{tang2025investigating}. By leveraging this spectral structure, our bound yields a practically computable sample complexity.

\section{Application}
In this section, we provide empirical evidence to support the theoretical findings of BSI. Our primary goal is to demonstrate that BSI yields improved performance in practical settings compared to existing low-rank compression techniques.
\subsection{Evaluation Setup}
\label{subsec:datasets_model}

Mathematical reasoning is one of the most challenging areas for LLM evaluation because it requires solving non-trivial problems. We evaluate the performance of low-rank compression algorithms across two distinct mathematical reasoning datasets: GSM8K \citep{gsm8k} and Hendrycks' MATH \citep{hendrycks}. These experiments are conducted using the Llama 2-7B~\citep{llama2}. We benchmark our method against four representative low-rank decomposition techniques. We consider SVD~\citep{LASER, SVD0, SVD1, SVD2, SVD3, SVD4}, FWSVD, which augments SVD with gradient-based profiling on the fine-tuning dataset during decomposition~\citep{hsu2022fwswd}, ASVD, which seeks a low-rank weight approximation that minimizes activation reconstruction loss~\citep{asvd}, and Basel, a state-of-the-art approach that exploits the semantic structure of weight matrices for compression~\citep{li2025basel}. To ensure a rigorous comparison, each model undergoes a three-stage pipeline: (1) task-specific fine-tuning of the base model, (2) low-rank compression via a candidate algorithm, and (3) post-compression fine-tuning for task adaptation.


\subsection{Evaluation on Mathematical Reasoning}
Figures~\ref{fig:7b-math} (a) and (b) show the performance of Llama-2-7B on mathematical reasoning tasks across different low-rank compression techniques. At low compression ratios\footnote{The compression ratio is defined as the quotient of the original model size divided by the compressed model size.} (below approximately 5×), all methods achieve comparable performance on both GSM8K and Hendrycks' MATH. As compression becomes more aggressive, however, a clear divergence emerges. SVD, ASVD and FWSVD degrade rapidly, with FWSVD exhibiting an early collapse in accuracy. In contrast, our BSI method degrades more gracefully and consistently preserves higher accuracy. 
On GSM8K, at 16$\times$ compression, the strongest baseline retains only about 9\% accuracy, whereas BSI maintains roughly 27\%. Moreover, on both GSM8K and Hendrycks’ MATH, BSI at a 19.5$\times$ compression ratio matches the accuracy of Basel at 13.8$\times$ on GSM8K and 14.9$\times$ on MATH. In other words, BSI supports up to about 1.3$\times$ higher compression than Basel while maintaining comparable accuracy.

\begin{figure}[t]
\centering
\begin{minipage}[b]{0.48\textwidth}
  \centering
  \includegraphics[width=\linewidth]{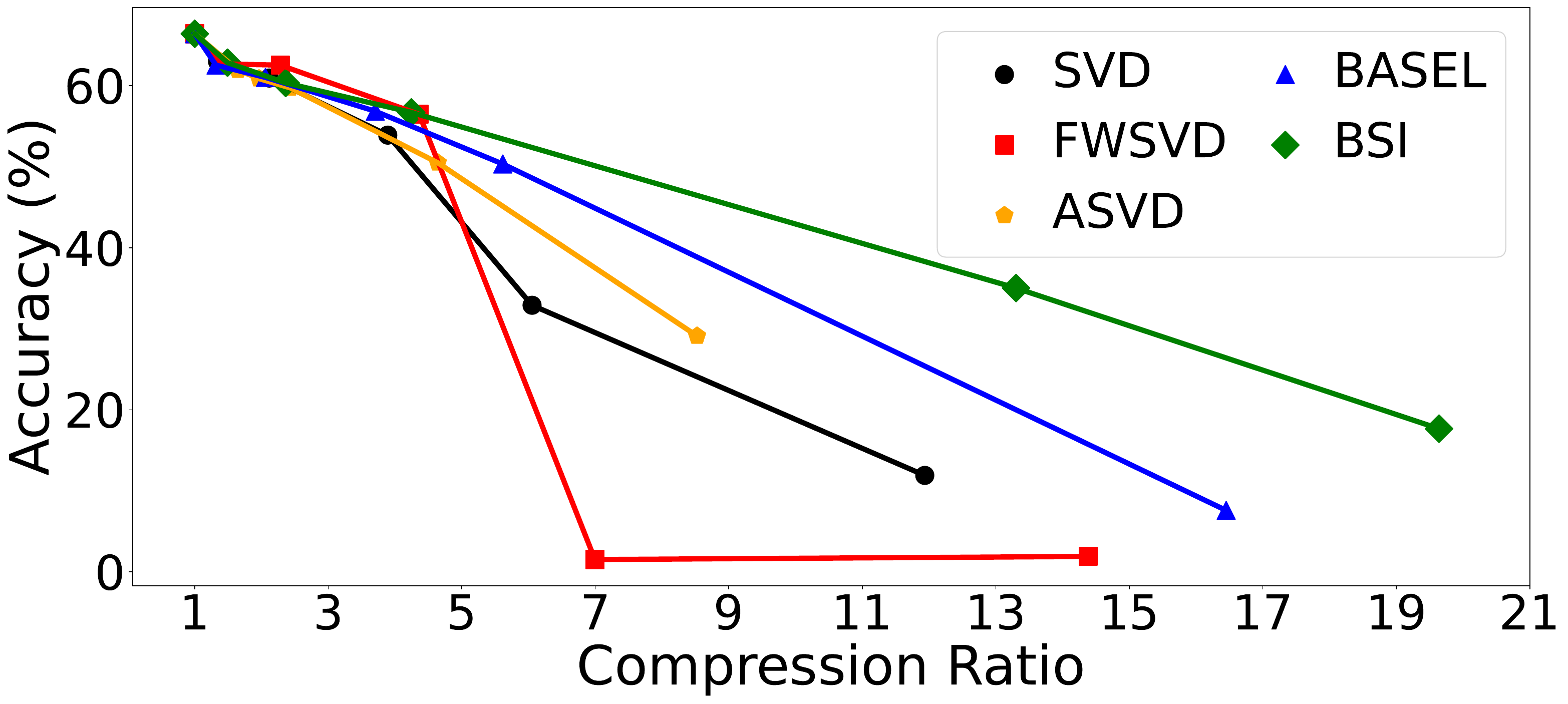}\\
  \small (a) GSM8K
\end{minipage}\hfill
\begin{minipage}[b]{0.48\textwidth}
  \centering
  \includegraphics[width=\linewidth]{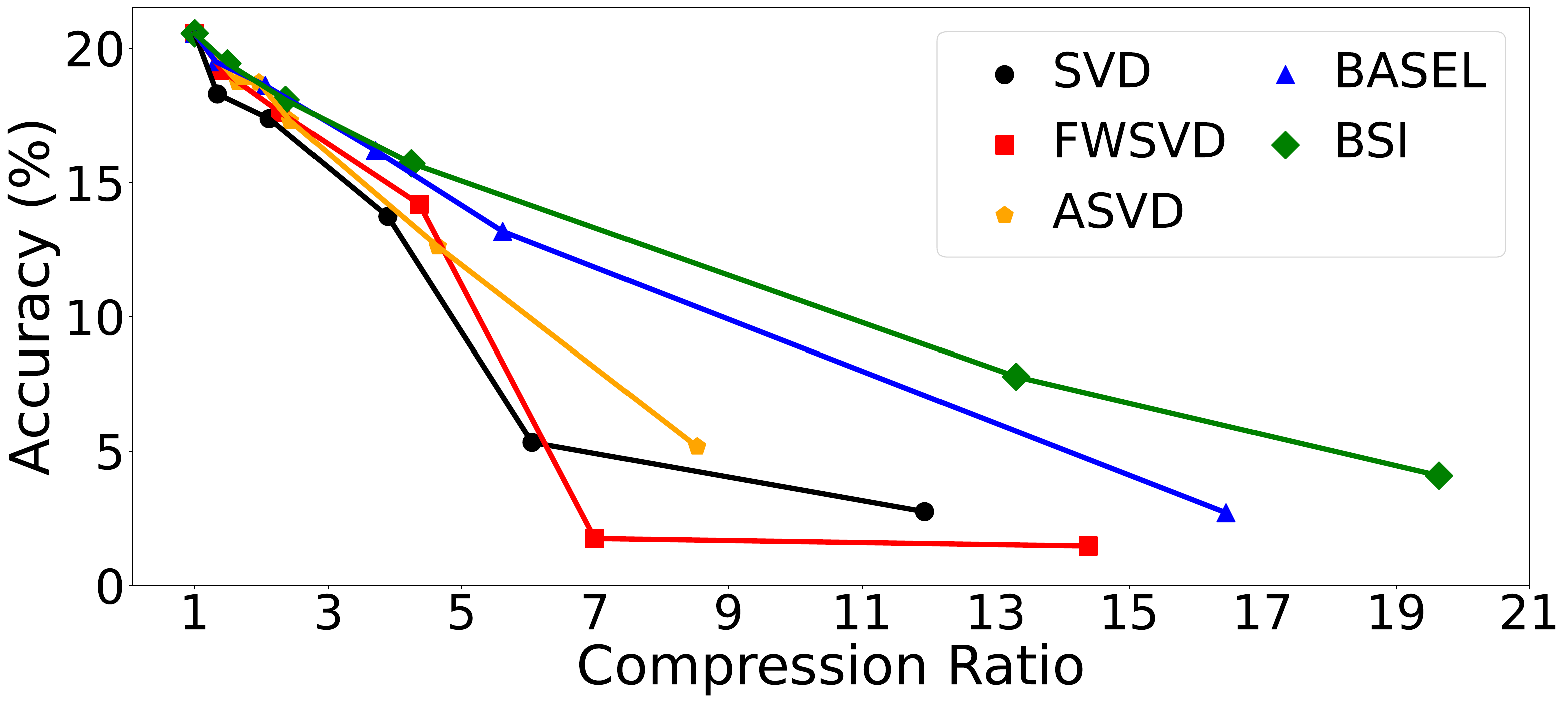}\\
  \small (b) MATH
\end{minipage}
\caption{Accuracy and model size of Llama 2-7B compressed with various low-rank decomposition algorithms on mathematical reasoning tasks.}
\label{fig:7b-math}
\end{figure}

\subsection{Ablation Study}
\label{sec:hessian_ablation}

The BSI importance metric used for basis selection includes both first-order gradient information and second-order Hessian information. To isolate the contribution of the Hessian term, we conduct an ablation study comparing the proposed BSI importance metric with a first-order variant that uses only the gradient term. Specifically, for basis \(i\), the gradient-only importance metric is defined as
\(
I_i \triangleq -\sigma_i\,\mathbb{E}\!\left[\frac{\partial \ell}{\partial \sigma_i}\right].
\) We perform this ablation study using Llama 2-7B on the Hendrycks' MATH and GSM8K datasets. The results are reported in Figure~\ref{fig:BSI_w_o_hessian_7b-math} (Appendix~\ref{app:hessian_ablation}). Across both datasets, the proposed BSI  consistently outperforms the gradient-only variant, demonstrating that the second-order term provides useful curvature information for selecting important singular bases. These results validate the role of the Hessian component in the proposed importance metric.



\section{Conclusion}


In this work, we set out to better understand the basis selection problem in low-rank compression---specifically, what is actually lost when individual SVD bases are removed. This motivation led to BSI, a framework that moves beyond magnitude-based pruning and instead grounds basis selection directly in the loss landscape. BSI selects and prunes bases using an importance score that estimates the expected change in loss caused by removing each basis. Our theoretical results show that the second-order information required by BSI---namely, the Hessian diagonal---can be estimated efficiently and reliably at scale. Leveraging the spectral structure of the Hessian (including power-law decay), we relate both the estimator's variance and its high-probability $(\varepsilon,\delta)$ error guarantee to the spectrum, thereby clarifying the sample complexity needed for accurate estimation. Beyond estimator analysis, we also bound the loss change induced by basis pruning and characterize how Hessian-diagonal estimation error propagates into this bound. Experiments on mathematical reasoning tasks demonstrate that BSI achieves a stronger compression--accuracy trade-off, compressing models by up to $1.3\times$ more than state-of-the-art low-rank baselines while maintaining similar performance.


\bibliographystyle{plainnat}
\bibliography{references}
\newpage
\noindent\textbf{Appendix Table of Contents}
\vspace{0.6em}

\noindent\textbf{Appendix A: Theoretical Analysis of Basis Selection with Importance}
\vspace{0.3em}

\noindent\begin{tabular}{@{}p{0.94\linewidth}r@{}}
\hspace{1.5em}A.1\quad Efficient Estimation of the Hessian Diagonal via Perturbations \dotfill
& \pageref{apdx:hessian_diagona_estimate_proof} \\

\hspace{1.5em}A.2\quad Bound on the Loss Change \dotfill
& \pageref{appendix:loss-bound-sec} \\

\hspace{1.5em}A.3\quad Variance of the Hessian Diagonal Estimator \dotfill
& \pageref{apdx:variance-dimensionality-convergence-proof} \\

\hspace{1.5em}A.4\quad Probabilistic Guarantee for Hessian Diagonal Estimation \dotfill
& \pageref{appendix:prob-bound-section} \\

\hspace{1.5em}A.5\quad Non-negativity of the Hessian Diagonal of the Output Layer \dotfill
& \pageref{appendix:proof-prop-lm-head}
\end{tabular}

\vspace{0.8em}

\noindent\textbf{Appendix B: Hessian Spectrum in the Singular-Value Space}
\leaders\hbox{.}\hfill
\pageref{sec:sv_hessian_spectrum}

\vspace{0.8em}

\noindent\textbf{Appendix C: BSI Practical Design Choices}
\vspace{0.3em}

\noindent\begin{tabular}{@{}p{0.94\linewidth}r@{}}
\hspace{1.5em}C.1\quad Choosing Perturbation Intensity \dotfill
& \pageref{appendix:perturbation-intensity} \\

\hspace{1.5em}C.2\quad Candidate Pool for Hessian Estimate \dotfill
& \pageref{sec:candidate_pool}
\end{tabular}

\vspace{0.8em}

\noindent\textbf{Appendix D: BSI Application}
\vspace{0.3em}

\noindent\begin{tabular}{@{}p{0.94\linewidth}r@{}}
\hspace{1.5em}D.1\quad Evaluation of BSI on Mathematical Reasoning \dotfill
& \pageref{appendix:math-results}
\end{tabular}
\newpage
\appendix

\section{Theoretical Analysis of Basis Selection with Importance
\label{appndix:theory_proofs}}
This section contains all proofs underlying the theoretical analysis of BSI. 
(i) We present an efficient perturbation-based estimator for the diagonal elements of the Hessian (Appendix~\ref{apdx:hessian_diagona_estimate_proof}). 
(ii) We derive bounds on the loss change induced by basis pruning in two settings: one with an exact Hessian and another with estimation error (Appendix~\ref{appendix:loss-bound-sec}). 
(iii) We establish a bound on the variance of the estimator under power-law spectral decay (Appendix~\ref{apdx:variance-dimensionality-convergence-proof}). 
(iv) We provide a high-probability bound on the sample complexity of the diagonal estimator, explicitly tied to the spectral structure of the Hessian (Appendix~\ref{appendix:prob-bound-section}), and further specialize this result to the positive semidefinite case (Appendix~\ref{appendix:prob-bound-subsection}). 
(v) Finally, we prove the non-negativity of the Hessian diagonal for the output layer of large language models (Appendix~\ref{appendix:proof-prop-lm-head}).

\subsection{Efficient Estimation of the Hessian Diagonal via Perturbations \label{apdx:hessian_diagona_estimate_proof}}

The importance score used in BSI contains a second-order term involving the
diagonal of the Hessian with respect to the singular values.
Explicit computation of this quantity is infeasible for large models due to
the prohibitive cost of forming or storing the full Hessian. Theorem~\ref{thm:hessian_diag_est} provides an efficient and reliable estimator for the
Hessian diagonal that relies only on gradient which can be computed efficiently for large models. The key idea is to use symmetric perturbations and exploit the fact that cross-coordinate interactions are canceled in expectation under Rademacher perturbations. This leaves only the diagonal of the Hessian, allowing us to estimate the Hessian diagonal efficiently using only gradient computations at perturbed values.

\HessDiagEstimate*
\begin{proof}
\begin{align*}
\text{RHS}
\;&=\;
\mathbb{E}_{\substack{\mathbf{z} \sim \mathcal{R}(\mathcal{Z}) \\ d \sim \mathcal{D}}}
\left[
\left(
\left.\frac{\partial^2 \ell(d)}{\partial \boldsymbol{\sigma}^2}\right|_{\boldsymbol{\sigma}=\boldsymbol{\sigma}^\ast}
(\epsilon\odot \mathbf{z}) + O(\epsilon^3)
\right)_i
\cdot
\frac{z_i}{\epsilon}
\right] \\
&=\;
\mathbb{E}_{\substack{\mathbf{z}\sim \mathcal{R}(\mathcal{Z})\\ d\sim \mathcal{D}}}
\left[
\left(
\sum_{j}
\left(
\left.\frac{\partial^{2}\ell(d)}{\partial \boldsymbol{\sigma}^{2}}\right|_{\boldsymbol{\sigma}=\boldsymbol{\sigma}^{\ast}}
\right)_{ij}
\, \epsilon \cdot z_j
\right)
\cdot
\frac{z_i}{\epsilon}
\right] 
\;+\; O(\epsilon^{2})\\
&=\;
\sum_{j}
\mathbb{E}_{\substack{\mathbf{z}\sim \mathcal{R}(\mathcal{Z})\\ d\sim \mathcal{D}}}
\left[
\left(
\left.\frac{\partial^{2}\ell(d)}{\partial \boldsymbol{\sigma}^{2}}\right|_{\boldsymbol{\sigma}=\boldsymbol{\sigma}^{\ast}}
\right)_{ij}
z_i z_j
\right] \;+\; O(\epsilon^{2}).
\end{align*}
Since $z_i$ and $z_j$ are i.i.d., for $i \neq j$
\[
\mathbb{E}[z_i z_j] = \mathbb{E}[z_i]\mathbb{E}[z_j] = 0
\]
\begin{align*}
\text{RHS}
\;&=\;
\mathbb{E}_{\substack{\mathbf{z}\sim \mathcal{R}(\mathcal{Z})\\ d\sim \mathcal{D}}}
\left[
\left(
\left.\frac{\partial^{2}\ell(d)}{\partial \boldsymbol{\sigma}^{2}}\right|_{\boldsymbol{\sigma}=\boldsymbol{\sigma}^{\ast}}
\right)_{ii}
\, z_i^{2}
\right]
\;+\; O(\epsilon^{2})\\
&=\;
\mathbb{E}_{d\sim \mathcal{D}}
\left[
\left(
\left.
\frac{\partial^{2}\ell(d)}{\partial \boldsymbol{\boldsymbol{\sigma}}^{2}}
\right|_{\boldsymbol{\sigma}=\sigma^{\ast}}
\right)_{ii}
\right] 
\;
\mathbb{E}_{\mathbf{z}\sim \mathcal{R}(\mathcal{Z})}[z_i^{2}]  \;+\; O(\epsilon^{2})\\
&=\;
\mathbb{E}_{d\sim \mathcal{D}}
\left[
\left(
\left.
\frac{\partial^{2}\ell(d)}{\partial \boldsymbol{\sigma}^{2}}
\right|_{\boldsymbol{\sigma}=\boldsymbol{\sigma}^{\ast}}
\right)_{ii}
\right] \;+\; O(\epsilon^{2}) \\
\;
&=\; \text{LHS}.
\end{align*}

\end{proof}

\subsection{Bound on the Loss Change\label{appendix:loss-bound-sec}}
The bounds on the loss change induced by pruning a basis presented in Theorems~\ref{thm:basis_pruning_bound_on_loss} and \ref{thm:basis_pruning_relative_hessian_error} rely on the following lemma:

\begin{lemma}[Bound on the Loss Change]
\label{lem:taylor_lipschitz_hessian}
Assume function $\ell:\mathbb{R}^r\to\mathbb{R}$ is $C^2$, and its Hessian is $\rho$-Lipschitz in operator norm on the line segment
$\{\mathbf{s}+t\boldsymbol{\Delta}\mid t\in[0,1]\}$:
\[
\left\|
\nabla^2 \ell(\mathbf{x})-\nabla^2 \ell(\mathbf{y})
\right\|_2
\le
\rho \left\|\mathbf{x}-\mathbf{y}\right\|_2
\qquad
\text{for all $\mathbf{x},\mathbf{y}$ on the segment.}
\]
Then
\(
\ell(\mathbf{s}+\boldsymbol{\Delta})
=
\ell(\mathbf{s})
+
\nabla \ell(\mathbf{s})^\top \boldsymbol{\Delta}
+
\frac{1}{2}\boldsymbol{\Delta}^\top \nabla^2 \ell(\mathbf{s})\,\boldsymbol{\Delta}
+
R(\boldsymbol{\Delta}),
\)
with the cubic remainder bound
\[
\left|R(\boldsymbol{\Delta})\right|
\le
\frac{\rho}{6}\left\|\boldsymbol{\Delta}\right\|_2^3.
\]
\end{lemma}

\begin{proof}
Define the one-dimensional restriction
\[
\Phi(t) := \ell(\mathbf{s}+t\boldsymbol{\Delta}), \qquad t\in[0,1].
\]
Then, by the chain rule,
\[
\Phi'(t) = \nabla \ell(\mathbf{s}+t\boldsymbol{\Delta})^\top \boldsymbol{\Delta},
\qquad
\Phi''(t) = \boldsymbol{\Delta}^\top \nabla^2 \ell(\mathbf{s}+t\boldsymbol{\Delta})\,\boldsymbol{\Delta}.
\]
Moreover,
\begin{align*}
\Phi(1)-\Phi(0)
&=\int_0^1 \Phi'(t)\,dt \\
&= \int_0^1 \left(\Phi'(0)+\int_0^t \Phi''(u)\,du\right) dt \\
&= \Phi'(0)\int_0^1 dt + \int_0^1\int_0^t \Phi''(u)\,du\,dt \\
&= \Phi'(0) + \int_0^1\int_u^1 \Phi''(u)\,dt\,du \\
&= \Phi'(0) + \int_0^1 (1-u)\,\Phi''(u)\,du \\
&= \Phi'(0) + \int_0^1 (1-t)\,\Phi''(t)\,dt.
\end{align*}
Therefore,
\begin{align*}
\Phi(1)
&= \Phi(0) + \Phi'(0) + \int_0^1 (1-t)\,\Phi''(t)\,dt \\
&= \Phi(0) + \Phi'(0) + \frac{1}{2}\Phi''(0)
   + \int_0^1 (1-t)\big(\Phi''(t)-\Phi''(0)\big)\,dt.
\end{align*}
Further because $\Phi(1)=\ell(\mathbf{s}+\boldsymbol{\Delta})$, $\Phi(0)=\ell(\mathbf{s})$, $
\Phi'(0)=\nabla \ell(\mathbf{s})^\top \boldsymbol{\Delta},
\Phi''(0)=\boldsymbol{\Delta}^\top \nabla^2 \ell(\mathbf{s})\,\boldsymbol{\Delta},
$
we have
\begin{align*}
\ell(\mathbf{s}+\boldsymbol{\Delta})
&= \ell(\mathbf{s})
+ \nabla \ell(\mathbf{s})^\top \boldsymbol{\Delta}
+ \frac{1}{2}\boldsymbol{\Delta}^\top \nabla^2 \ell(\mathbf{s})\,\boldsymbol{\Delta} \\
&\quad + \int_0^1 (1-t)\Big(
\boldsymbol{\Delta}^\top \nabla^2 \ell(\mathbf{s}+t\boldsymbol{\Delta})\,\boldsymbol{\Delta}
-
\boldsymbol{\Delta}^\top \nabla^2 \ell(\mathbf{s})\,\boldsymbol{\Delta}
\Big)\,dt.
\end{align*}
Let
\begin{align*}
R(\boldsymbol{\Delta})
&:= \int_0^1 (1-t)\Big(
\boldsymbol{\Delta}^\top \nabla^2 \ell(\mathbf{s}+t\boldsymbol{\Delta})\,\boldsymbol{\Delta}
-
\boldsymbol{\Delta}^\top \nabla^2 \ell(\mathbf{s})\,\boldsymbol{\Delta}
\Big)\,dt \\
&= \int_0^1 (1-t)\,
\boldsymbol{\Delta}^\top\Big(\nabla^2 \ell(\mathbf{s}+t\boldsymbol{\Delta})-\nabla^2 \ell(\mathbf{s})\Big)\boldsymbol{\Delta}\,dt.
\end{align*}
Then
\begin{align*}
\left|R(\boldsymbol{\Delta})\right|
&= \left|\int_0^1 (1-t)\,
\boldsymbol{\Delta}^\top\Big(\nabla^2 \ell(\mathbf{s}+t\boldsymbol{\Delta})-\nabla^2 \ell(\mathbf{s})\Big)\boldsymbol{\Delta}\,dt\right| \\
&\le \int_0^1 (1-t)\,
\left|\boldsymbol{\Delta}^\top\Big(\nabla^2 \ell(\mathbf{s}+t\boldsymbol{\Delta})-\nabla^2 \ell(\mathbf{s})\Big)\boldsymbol{\Delta}\right|\,dt.
\end{align*}
Based on the Cauchy--Schwarz inequality and the definition of operator norm,
\[
\left|\boldsymbol{\Delta}^\top\Big(\nabla^2 \ell(\mathbf{s}+t\boldsymbol{\Delta})-\nabla^2 \ell(\mathbf{s})\Big)\boldsymbol{\Delta}\right|
\le
\|\boldsymbol{\Delta}\|_2^2\,
\left\|\nabla^2 \ell(\mathbf{s}+t\boldsymbol{\Delta})-\nabla^2 \ell(\mathbf{s})\right\|_2,
\]
and therefore
\[
\left|R(\boldsymbol{\Delta})\right|
\le
\int_0^1 (1-t)\,\|\boldsymbol{\Delta}\|_2^2\,
\left\|\nabla^2 \ell(\mathbf{s}+t\boldsymbol{\Delta})-\nabla^2 \ell(\mathbf{s})\right\|_2\,dt.
\]
Further based on Hessian Lipschitzness,
\[
\left\|\nabla^2 \ell(\mathbf{s}+t\boldsymbol{\Delta})-\nabla^2 \ell(\mathbf{s})\right\|_2
\le
\rho \left\|t\boldsymbol{\Delta}\right\|_2
=
\rho t \|\boldsymbol{\Delta}\|_2,
\]
So,
\begin{align*}
\left|R(\boldsymbol{\Delta})\right|
&\le
\int_0^1 (1-t)\,\|\boldsymbol{\Delta}\|_2^2\,
\big(\rho t \|\boldsymbol{\Delta}\|_2\big)\,dt \\
&=
\rho \|\boldsymbol{\Delta}\|_2^3 \int_0^1 (t-t^2)\,dt \\
&=
\frac{\rho}{6}\|\boldsymbol{\Delta}\|_2^3.
\end{align*}
Therefore,
\[
\ell(\mathbf{s}+\boldsymbol{\Delta})
=
\ell(\mathbf{s})
+ \nabla \ell(\mathbf{s})^\top \boldsymbol{\Delta}
+ \frac{1}{2}\boldsymbol{\Delta}^\top \nabla^2 \ell(\mathbf{s})\,\boldsymbol{\Delta}
+ R(\boldsymbol{\Delta}), \text{ with }
\left|R(\boldsymbol{\Delta})\right|
\le 
\frac{\rho}{6}\|\boldsymbol{\Delta}\|_2^3.
\]
\end{proof}

\subsubsection{Bound for Loss Change Induced by Pruning a Basis \label{appendix:proof-thm-loss-bound-on-basis-pruning}}
We quantify the impact of basis pruning on the change in loss. Theorem~\ref{thm:basis_pruning_bound_on_loss} presents an upper bound on the loss change resulting from the removal of a basis in the setting where the Hessian diagonal is exact. In practice, however, the Hessian diagonal is obtained via estimation. We therefore extend the bound for the loss change to account for relative error in the Hessian diagonal estimate (Theorem~\ref{thm:basis_pruning_relative_hessian_error}).
\LossBoundGeneral*

\begin{proof}
Let $\boldsymbol{\Delta}=-s_i\mathbf{e}_i$. Based on Lemma~\ref{lem:taylor_lipschitz_hessian},
\[
\ell(\mathbf{s}-s_i\mathbf{e}_i)
=
\ell(\mathbf{s})
+
\nabla \ell(\mathbf{s})^\top(-s_i\mathbf{e}_i)
+
\frac{1}{2}(-s_i\mathbf{e}_i)^\top \nabla^2 \ell(\mathbf{s})\,(-s_i\mathbf{e}_i)
+
R(-s_i\mathbf{e}_i),
\]
with
\[
\left|R(-s_i\mathbf{e}_i)\right|
\le
\frac{\rho}{6}\left\|-s_i\mathbf{e}_i\right\|_2^3
=
\frac{\rho}{6}|s_i|^3.
\]
Therefore,
\begin{align*}
|\Delta \ell|
&=
\big|\ell(\mathbf{s}-s_i\mathbf{e}_i)-\ell(\mathbf{s})\big| \\
&=
\left|
\nabla \ell(\mathbf{s})^\top(-s_i\mathbf{e}_i)
+
\frac{1}{2}(-s_i\mathbf{e}_i)^\top \nabla^2 \ell(\mathbf{s})\,(-s_i\mathbf{e}_i)
+
R(-s_i\mathbf{e}_i)
\right|.
\end{align*}
As $g_i=\nabla \ell(\mathbf{s})^\top \mathbf{e}_i$ and $H_{ii}=\mathbf{e}_i^\top \nabla^2 \ell(\mathbf{s})\,\mathbf{e}_i$, we have
\begin{align*}
|\Delta \ell|
&=
\left|-s_i g_i + \frac{1}{2}s_i^2 H_{ii} + R(-s_i\mathbf{e}_i)\right| \\
&\le
\left|-s_i g_i + \frac{1}{2}s_i^2 H_{ii}\right|
+
\left|R(-s_i\mathbf{e}_i)\right| \\
&\le
\left|-s_i g_i + \frac{1}{2}s_i^2 H_{ii}\right|
+
\frac{\rho}{6}|s_i|^3 .
\end{align*}
\end{proof}

\LossBoundRelativeError*

\begin{proof}
\label{appendix:proof-them-bound-on-loss-with-rel-error}
We begin by bounding $H_{ii}$ in terms of $\widehat H_{ii}$. We write
\begin{align*}
\lvert H_{ii} \rvert
&=
\lvert (H_{ii} - \widehat H_{ii}) + \widehat H_{ii} \rvert \\
&\le
\lvert H_{ii} - \widehat H_{ii} \rvert
+
\lvert \widehat H_{ii} \rvert \\
&=
\lvert \widehat H_{ii} - H_{ii} \rvert
+
\lvert \widehat H_{ii} \rvert \\
&\le
\varepsilon \lvert H_{ii} \rvert
+
\lvert \widehat H_{ii} \rvert .
\end{align*}
Rearranging gives
\begin{align*}
(1-\varepsilon)\lvert H_{ii} \rvert
&\le
\lvert \widehat H_{ii} \rvert ,
\end{align*}
Hence,
\begin{align*}
\lvert H_{ii} \rvert
&\le
\frac{1}{1-\varepsilon}
\lvert \widehat H_{ii} \rvert .
\end{align*}
From Theorem~\ref{thm:basis_pruning_bound_on_loss},
\begin{align*}
\lvert \Delta \ell \rvert
&\le
\left|
- g_i s_i
+ \frac{1}{2} H_{ii} s_i^2
\right|
+
\frac{\rho}{6}\lvert s_i \rvert^3 .
\end{align*}
We insert and subtract $\tfrac{1}{2}\widehat H_{ii} s_i^2$
\begin{align*}
\lvert \Delta \ell \rvert
&\le
\left|
- g_i s_i
+ \frac{1}{2}\widehat H_{ii} s_i^2
+ \frac{1}{2} H_{ii} s_i^2
- \frac{1}{2}\widehat H_{ii} s_i^2
\right|
+
\frac{\rho}{6}\lvert s_i \rvert^3 .
\end{align*}
Applying the triangle inequality yields
\begin{align*}
\lvert \Delta \ell \rvert
&\le
\left|
- g_i s_i
+ \frac{1}{2}\widehat H_{ii} s_i^2
\right|
+
\frac{1}{2}s_i^2
\lvert H_{ii} - \widehat H_{ii} \rvert
+
\frac{\rho}{6}\lvert s_i \rvert^3 .
\end{align*}
As $\widehat H_{ii}$ has a maximum relative error of $\varepsilon$,
\begin{align*}
\lvert \Delta \ell \rvert
&\le
\left|
- g_i s_i
+ \frac{1}{2}\widehat H_{ii} s_i^2
\right|
+
\frac{1}{2}\varepsilon s_i^2
\lvert H_{ii} \rvert
+
\frac{\rho}{6}\lvert s_i \rvert^3 .
\end{align*}
 As $\lvert H_{ii} \rvert \le \frac{1}{1-\varepsilon}\, \lvert \widehat{H}_{ii} \rvert,$

\begin{align*}
\lvert \Delta \ell \rvert
&\le
\left|
- g_i s_i
+ \frac{1}{2}\widehat H_{ii} s_i^2
\right|
+
\frac{\varepsilon}{2(1-\varepsilon)}
\lvert \widehat H_{ii} \rvert s_i^2
+
\frac{\rho}{6}\lvert s_i \rvert^3 .
\end{align*}
\end{proof}

\subsection{Variance of the Hessian Diagonal Estimator \label{apdx:variance-dimensionality-convergence-proof}}
By leveraging the power-law decay of neural network Hessian spectrum (Assumption~\ref{assump:spectral-power-law}), we prove that the total variance of the Hessian diagonal estimator is bounded independently of the parameter dimension (Theorem~\ref{thm:hutchinson-variance-dimensionality}). 


 \VarDimensionalityHessian*

\begin{proof}
We begin by expanding the Hutchinson diagonal estimator. The $i$-th component is
\begin{align*}
\mathrm{D}^{(s)}(\mathbf{H})_i
&= \frac{1}{s}\sum_{k=1}^s z^{(k)}_i \, (\mathbf{H}\mathbf{z}^{(k)})_i \\
&= \frac{1}{s}\sum_{k=1}^s \sum_{j=1}^n H_{ij}\, z^{(k)}_i z^{(k)}_j \\
&= \frac{1}{s}\sum_{k=1}^s \left( H_{ii} (z^{(k)}_i)^2 + \sum_{j \neq i} H_{ij} z^{(k)}_i z^{(k)}_j \right) \\
&= H_{ii} + \frac{1}{s}\sum_{k=1}^s \sum_{j \neq i} H_{ij} z^{(k)}_i z^{(k)}_j .
\end{align*}
Since $\mathbb{E}[z^{(k)}_i z^{(k)}_j]=0$ for $j\neq i$, we have
\[
\mathbb{E}\!\left[\mathrm{D}^{(s)}(\mathbf{H})_i\right]
=
H_{ii}
+
\frac{1}{s}\sum_{k=1}^s \sum_{j \neq i} H_{ij}\, \mathbb{E}[z^{(k)}_i z^{(k)}_j]
=
H_{ii},
\]
Using $\mathrm{Var}(X)=\mathbb{E}[X^2]-(\mathbb{E}[X])^2$ with
$X=\mathrm{D}^{(s)}(\mathbf{H})_i$, we write
\begin{align*}
\mathrm{Var}\!\big(\mathrm{D}^{(s)}(\mathbf{H})_i\big)
&=
\mathbb{E}\!\left[
\left(
H_{ii}
+
\frac{1}{s}\sum_{k=1}^s \sum_{j \neq i} H_{ij}\, z^{(k)}_i z^{(k)}_j
\right)^2
\right]
-
H_{ii}^2 \\[6pt]
&=
\mathbb{E}\!\left[
H_{ii}^2
+
\frac{2H_{ii}}{s}\sum_{k=1}^s \sum_{j \neq i} H_{ij}\, z^{(k)}_i z^{(k)}_j
+
\left(
\frac{1}{s}\sum_{k=1}^s \sum_{j \neq i} H_{ij}\, z^{(k)}_i z^{(k)}_j
\right)^2
\right]
-
H_{ii}^2 . \\
&=
H_{ii}^2
+
\frac{2H_{ii}}{s}\sum_{k=1}^s \sum_{j \neq i} H_{ij}\,
\mathbb{E}[z^{(k)}_i z^{(k)}_j]
+
\mathbb{E}\!\left[
\left(
\frac{1}{s}\sum_{k=1}^s \sum_{j \neq i} H_{ij}\, z^{(k)}_i z^{(k)}_j
\right)^2
\right]
-
H_{ii}^2 .
\end{align*}
Since $\mathbb{E}[z^{(k)}_i z^{(k)}_j]=0$ for $j\neq i$, we have
\begin{align*}
\mathrm{Var}\!\big(\mathrm{D}^{(s)}(\mathbf{H})_i\big)
=&
\mathbb{E}\!\left[
\left(
\frac{1}{s}\sum_{k=1}^s \sum_{j \neq i} H_{ij}\, z^{(k)}_i z^{(k)}_j
\right)^2
\right] \\[6pt]
=&
\frac{1}{s^2}\,
\mathbb{E}\!\left[
\sum_{k=1}^s
\left(\sum_{j \neq i} H_{ij}\, z^{(k)}_i z^{(k)}_j\right)^2
\right]
 \\[6pt]
&+\frac{1}{s^2}\,
\mathbb{E}\!\left[
\sum_{m\neq n}
\left(\sum_{j \neq i} H_{ij}\, z^{(m)}_i z^{(m)}_j\right)
\left(\sum_{j \neq i} H_{ij}\, z^{(n)}_i z^{(n)}_j\right)
\right] \\[6pt]
=&
\frac{1}{s^2}\,
\mathbb{E}\!\left[
\sum_{k=1}^s
\left(\sum_{j \neq i} H_{ij}\, z^{(k)}_i z^{(k)}_j\right)^2
\right] \\[6pt]
&+
\frac{1}{s^2}\,
\sum_{m\neq n}
\mathbb{E}\!\left[\sum_{j \neq i} H_{ij}\, z^{(m)}_i z^{(m)}_j\right]\,
\mathbb{E}\!\left[\sum_{j \neq i} H_{ij}\, z^{(n)}_i z^{(n)}_j\right].
\end{align*}
For Rademacher probes, \(
\mathbb{E}\!\big[z^{(k)}_i z^{(k)}_j\big]
=0
\) for $j\neq i$.
Therefore,
\[
\mathbb{E}\!\left[\sum_{j \neq i} H_{ij}\, z^{(k)}_i z^{(k)}_j\right]
=
\sum_{j\neq i} H_{ij}\,\mathbb{E}\!\big[z^{(k)}_i z^{(k)}_j\big]
=0.
\]
As a consequence, the entire $m\neq n$ contribution is zero, and we are left with
\begin{align*}
\mathrm{Var}\!\big(\mathrm{D}^{(s)}(\mathbf{H})_i\big)
&=
\frac{1}{s^2}\,
\mathbb{E}\!\left[
\sum_{k=1}^s
\left(\sum_{j \neq i} H_{ij}\, z^{(k)}_i z^{(k)}_j\right)^2
\right] \\[6pt]
&=
\frac{1}{s^2}\,
\sum_{k=1}^s
\mathbb{E}\!\left[
\left(\sum_{j \neq i} H_{ij}\, z^{(k)}_i z^{(k)}_j\right)^2
\right]\\
&=
\frac{1}{s^2}\,
\sum_{k=1}^s
\Bigg(
\mathbb{E}\!\left[
\sum_{j \neq i} H_{ij}^2\,\big(z^{(k)}_i\big)^2 \big(z^{(k)}_j\big)^2
\right] \\[4pt]
&\qquad\qquad\qquad
+
\mathbb{E}\!\left[
\sum_{\substack{j_1\neq i,\; j_2\neq i\\ j_1\neq j_2}}
H_{ij_1}H_{ij_2}\,\big(z^{(k)}_i\big)^2\, z^{(k)}_{j_1} z^{(k)}_{j_2}
\right] \\[4pt]
&\qquad\qquad\qquad
+
\mathbb{E}\!\left[
\sum_{\substack{i_1\neq j,\; i_2\neq j\\ i_1\neq i_2}}
H_{i_1 j}H_{i_2 j}\, z^{(k)}_{i_1} z^{(k)}_{i_2}\,\big(z^{(k)}_j\big)^2
\right] \\[4pt]
&\qquad\qquad\qquad
+
\mathbb{E}\!\left[
\sum_{\substack{i_1\neq i_2,\; j_1\neq j_2\\ i_1\neq j_1,\; i_2\neq j_2}}
H_{i_1 j_1}H_{i_2 j_2}\, z^{(k)}_{i_1} z^{(k)}_{j_1}\, z^{(k)}_{i_2} z^{(k)}_{j_2}
\right]
\Bigg).
\end{align*}
For Rademacher probes, $(z_i^{(k)})^2=1$, Hence
\[
\mathbb{E}\!\left[\big(z^{(k)}_i\big)^2 \big(z^{(k)}_j\big)^2\right]= \mathbb{E}\!\left[(z^{(k)}_i)^2\right]\,
\mathbb{E}\!\left[(z^{(k)}_j)^2\right] = 1,
\qquad
\mathbb{E}\!\left[\big(z^{(k)}_i\big)^2 z^{(k)}_{j_1} z^{(k)}_{j_2}\right]=0 \ (j_1\neq j_2),
\]
\[
\mathbb{E}\!\left[z^{(k)}_{i_1} z^{(k)}_{i_2}\big(z^{(k)}_j\big)^2\right]=0 \ (i_1\neq i_2),
\qquad
\mathbb{E}\!\left[z^{(k)}_{i_1} z^{(k)}_{j_1} z^{(k)}_{i_2} z^{(k)}_{j_2}\right]=0
\ \text{(all indices distinct).}
\]
\begin{align*}
\intertext{So only the first (square) term remains, and we get}
\mathrm{Var}\!\big(\mathrm{D}^{(s)}(\mathbf{H})_i\big)
&=
\frac{1}{s^2}\,
\sum_{k=1}^s
\sum_{j \neq i} H_{ij}^2
=
\frac{1}{s}\sum_{j \neq i} H_{ij}^2.
\end{align*}
Therefore,
\begin{align*}
\sum_{i=1}^n \mathrm{Var}\!\big(\mathrm{D}^{(s)}(\mathbf{H})_i\big)
= \frac{1}{s}\sum_{i=1}^n \sum_{j \neq i} H_{ij}^2 
\le \frac{1}{s}\sum_{i=1}^n \sum_{j=1}^n H_{ij}^2 
= \frac{1}{s}\|\mathbf{H}\|_F^2 .
\end{align*}
Moreover, since $\mathbf{H}$ is symmetric, we have
\begin{align*}
\|\mathbf{H}\|_F^2
&= \mathrm{Tr}(\mathbf{H}^\top \mathbf{H}) = \mathrm{Tr}(\mathbf{H}^2),
\end{align*}
Then
\[
\sum_{i=1}^n \mathrm{Var}\!\big(\mathrm{D}^{(s)}(\mathbf{H})_i\big) \leq \frac{1}{s}\mathrm{Tr}(\mathbf{H}^2)
\]
From spectral theory (Lemma~\ref{lem:entrywise-spectral}), $\mathbf{H}$ admits an
orthonormal eigen-decomposition:
\[
\mathbf{H} = \sum_{k=1}^n \lambda_k \mathbf{u}_k \mathbf{u}_k^\top,
\]
Then,
\begin{align*}
\mathbf{H}^2
&=
\left(\sum_{k=1}^n \lambda_k \mathbf{u}_k \mathbf{u}_k^\top\right)
\left(\sum_{\ell=1}^n \lambda_\ell \mathbf{u}_\ell \mathbf{u}_\ell^\top\right) 
=
\sum_{k,\ell=1}^n
\lambda_k \lambda_\ell\,
\mathbf{u}_k
(\mathbf{u}_k^\top \mathbf{u}_\ell)
\mathbf{u}_\ell^\top .
\end{align*}
Since $\{\mathbf{u}_k\}_{k=1}^n$ is a set of orthonormal bases,
$\mathbf{u}_k^\top \mathbf{u}_\ell = 0$ for $k \neq \ell$ and
$\mathbf{u}_k^\top \mathbf{u}_k = 1$, which yields
\begin{align*}
\mathbf{H}^2
&= \sum_{k=1}^n \lambda_k^2 \mathbf{u}_k \mathbf{u}_k^\top .
\end{align*}
We obtain
\begin{align*}
\mathrm{Tr}(\mathbf{H}^2)
=
\mathrm{Tr}\!\left(\sum_{k=1}^n \lambda_k^2 \mathbf{u}_k \mathbf{u}_k^\top\right) 
=
\sum_{k=1}^n \lambda_k^2
\mathrm{Tr}(\mathbf{u}_k \mathbf{u}_k^\top) 
=
\sum_{k=1}^n \lambda_k^2
\mathrm{Tr}(\mathbf{u}_k^\top \mathbf{u}_k) 
=
\sum_{k=1}^n \lambda_k^2 .
\end{align*}
Combining the above yields
\begin{equation}
\label{eqn-variance-in-egienvectors}
\sum_{i=1}^n \mathrm{Var}\!\big(\mathrm{D}^{(s)}(\mathbf{H})_i\big)
\le \frac{1}{s}\mathrm{Tr}(\mathbf{H}^2)
= \frac{1}{s}\sum_{k=1}^n \lambda_k^2 .
\end{equation}
Under the spectral power-law assumption (Assumption~\ref{assump:spectral-power-law}),
$|\lambda_k| \le |\lambda_1| k^{-\alpha}$, and hence $\lambda_k^2 \le |\lambda_1|^2 k^{-2\alpha}$.
Substituting into \eqref{eqn-variance-in-egienvectors} gives
\begin{align*}
\sum_{i=1}^n \mathrm{Var}\!\big(\mathrm{D}^{(s)}(\mathbf{H})_i\big)
\le \frac{1}{s}|\lambda_1|^2 \sum_{k=1}^n \frac{1}{k^{2\alpha}}
= \frac{1}{s}|\lambda_1|^2 H_{n,2\alpha}
\le \frac{1}{s}|\lambda_1|^2 \zeta(2\alpha).
\end{align*}
As $\alpha > \tfrac{1}{2}$, the series converges as $n \to \infty$. This bound for the total variance is independent of parameter dimensionality $n$.


\end{proof}

\subsection{Probabilistic Guarantee for Hessian Diagonal Estimation \label{appendix:prob-bound-section}}

We strengthen the variance analysis of the Hutchinson diagonal estimator by deriving a probabilistic bound for the estimation error. Our result, which is presented in Theorem~\ref{theorem:prob-hessian-diagonal}, is connected to the spectral structure of the Hessian. The proof of Theorem~\ref{theorem:prob-hessian-diagonal} uses the following Lemma:

\begin{lemma}
\label{theorem:eps_delta-approx}
For any $\varepsilon > 0$, $0 < \delta < 1$, and matrix $\mathbf{A} \in \mathbb{R}^{n \times n}$ with $\operatorname{diag}(\mathbf{A}) \neq \mathbf{0}$, the Rademacher diagonal estimator $\mathrm{D}^{(s)}(\mathbf{A})$ is an $(\varepsilon,\delta)$-approximator if the number of samples $s$ satisfies
\[
s
>
2\left(
\frac{\|\mathbf{A}\|_F^2 - \|\operatorname{diag}(\mathbf{A})\|_2^2}
{\|\operatorname{diag}(\mathbf{A})\|_2^2}
\right)
\left(
\frac{\ln(2n/\delta)}{\varepsilon^2}
\right).
\]
\end{lemma}

\begin{proof}
~\citet{baston2022stochastic} proposed the lemma, which we restate here.
\end{proof}

\ProbHessianDiagonal*

\begin{proof}
We begin with the decomposition
\begin{align*}
\mathbf{H}
&= \sum_{i=1}^n \lambda_i \mathbf{u}_i \mathbf{u}_i^\top .
\end{align*}
By definition of the Frobenius norm,
\begin{align*}
\|\mathbf{H}\|_F^2
&= \operatorname{Tr}(\mathbf{H}\mathbf{H}^\top).
\end{align*}
Expanding $\mathbf{H}\mathbf{H}^\top$, we obtain
\begin{align*}
\mathbf{H}\mathbf{H}^\top
&=
\left(\sum_{i=1}^n \lambda_i \mathbf{u}_i \mathbf{u}_i^\top\right)
\left(\sum_{j=1}^n \lambda_j \mathbf{u}_j \mathbf{u}_j^\top\right)^\top 
=
\sum_{i,j=1}^n
\lambda_i \lambda_j\,
\mathbf{u}_i
(\mathbf{u}_i^\top \mathbf{u}_j)
\mathbf{u}_j^\top .
\end{align*}
Since $\{\mathbf{u}_i\}_{i=1}^n$ is a set of orthonormal bases,
$\mathbf{u}_i^\top \mathbf{u}_j = 0$ for $i \neq j$ and
$\mathbf{u}_i^\top \mathbf{u}_i = 1$, which yields
\begin{align*}
\mathbf{H}\mathbf{H}^\top
&= \sum_{i=1}^n \lambda_i^2 \mathbf{u}_i \mathbf{u}_i^\top .
\end{align*}
Therefore,
\begin{align*}
\|\mathbf{H}\|_F^2
= \operatorname{Tr}(\mathbf{H}\mathbf{H}^\top) 
= \operatorname{Tr}\!\left(\sum_{i=1}^n \lambda_i^2 \mathbf{u}_i \mathbf{u}_i^\top \right) 
= \sum_{i=1}^n \lambda_i^2 \operatorname{Tr}(\mathbf{u}_i \mathbf{u}_i^\top) 
= \sum_{i=1}^n \lambda_i^2
\operatorname{Tr}(\mathbf{u}_i^\top \mathbf{u}_i) 
= \sum_{i=1}^n \lambda_i^2 .
\end{align*}
Under the power-law decay condition,
\begin{align*}
|\lambda_i|
&\le \frac{|\lambda_1|}{i^\alpha},
\end{align*}
Hence,
\begin{align*}
\lambda_i^2
&\le \frac{|\lambda_1|^2}{i^{2\alpha}} .
\end{align*}
Consequently,
\begin{align*}
\sum_{i=1}^n \lambda_i^2
\le |\lambda_1|^2 \sum_{i=1}^n \frac{1}{i^{2\alpha}} 
= |\lambda_1|^2 H_{n,2\alpha}.
\end{align*}
Therefore,
\begin{equation}
\label{eqn:frob}
    \|\mathbf{H}\|_F^2 \;\le\; |\lambda_1|^2 \, H_{n,2\alpha}
\end{equation} 
Next, we bound the diagonal term. By Cauchy--Schwarz,
\begin{align*}
\|\operatorname{diag}(\mathbf{H})\|_2^2
= \sum_{i=1}^n H_{ii}^2 
\ge \frac{1}{n}
\left(\sum_{i=1}^n H_{ii}\right)^2 
= \frac{1}{n}\operatorname{Tr}(\mathbf{H})^2 .
\end{align*}
Hence, 
\begin{equation}
\label{eqn:diag-bound}
    \|\operatorname{diag}(\mathbf{H})\|_2^2 \;\ge\; \frac{1}{n}\operatorname{Tr}(\mathbf{H})^2 .
\end{equation}
Combining  Eq.~(\ref{eqn:frob}) and Eq.~(\ref{eqn:diag-bound}),
\begin{align*}
\frac{\|\mathbf{H}\|_F^2 - \|\operatorname{diag}(\mathbf{H})\|_2^2}
{\|\operatorname{diag}(\mathbf{H})\|_2^2}
&=
\frac{\|\mathbf{H}\|_F^2}{\|\operatorname{diag}(\mathbf{H})\|_2^2} - 1 \\
&\leq
\frac{|\lambda_1|^2 H_{n,2\alpha}}
{\tfrac{1}{n}\operatorname{Tr}(\mathbf{H})^2} - 1 .
\end{align*}
From Lemma~\ref{theorem:eps_delta-approx},
\begin{align*}
\Pr\!\left(
\left\| \mathrm{D}^{(s)}(\mathbf{H}) - \operatorname{diag}(\mathbf{H}) \right\|_2
\le
\varepsilon \left\| \operatorname{diag}(\mathbf{H}) \right\|_2
\right)
&\ge 1 - \delta
\end{align*}
provided that
\begin{align*}
s
&>
2\left(
\frac{\|\mathbf{H}\|_F^2 - \|\operatorname{diag}(\mathbf{H})\|_2^2}
{\|\operatorname{diag}(\mathbf{H})\|_2^2}
\right)
\left(
\frac{\ln(2n/\delta)}{\varepsilon^2}
\right).
\end{align*}
Therefore, if 
\begin{align*}
s
&>
2\left(
\frac{|\lambda_1|^2 H_{n,2\alpha}}
{\tfrac{1}{n}\operatorname{Tr}(\mathbf{H})^2}
- 1
\right)
\left(
\frac{\ln(2n/\delta)}{\varepsilon^2}
\right)
\;\ge\;
2\left(
\frac{\|\mathbf{H}\|_F^2 - \|\operatorname{diag}(\mathbf{H})\|_2^2}
{\|\operatorname{diag}(\mathbf{H})\|_2^2}
\right)\left(
\frac{\ln(2n/\delta)}{\varepsilon^2}
\right),
\end{align*}
we can guarantee that
\begin{align*}
\Pr\!\left(
\left\| \mathrm{D}^{(s)} - \operatorname{diag}(\mathbf{H}) \right\|_2
\le
\varepsilon \left\| \operatorname{diag}(\mathbf{H}) \right\|_2
\right)
&\ge 1 - \delta.
\end{align*}
\end{proof}

\subsubsection{Probabilistic Bound for Positive semidefinite Hessians \label{appendix:prob-bound-subsection}}

When the Hessian is positive semidefinite, the general probabilistic bound
derived in Theorem~\ref{theorem:prob-hessian-diagonal} can be
simplified.
In this setting, the trace of the Hessian admits a direct lower bound in terms
of the leading eigenvalue, yielding a simplified condition
on the number of samples required.
Corollary~\ref{corollary:prob-bound-PSD} provides an explicit $(\varepsilon,\delta)$ probabilistic bound on the sample
complexity for Hessian-diagonal estimation in the positive semidefinite case.

\ProbBoundPSDHessian*

\begin{proof}
As $\mathbf{H}$ is positive semidefinite, $\lambda_i \ge 0$, $i=1,\ldots,n$.
Further because $\lambda_1 > 0$, 
\[
\operatorname{Tr}(\mathbf{H})=\sum_{i=1}^n \lambda_i \ge \lambda_1 > 0.
\]
Based on Theorem~\ref{theorem:prob-hessian-diagonal}, the Rademacher diagonal estimator $\mathrm{D}^{(s)}$ is a relative $(\varepsilon,\delta)$-approximator if
\[
s
>
2\left(
\frac{|\lambda_1^2| H_{n,2\alpha}}{\tfrac{1}{n}\operatorname{Tr}(\mathbf{H})^2}
- 1
\right)
\left(
\frac{\ln(2n/\delta)}{\varepsilon^2}
\right).
\]
As $\operatorname{Tr}(\mathbf{H}) \ge \lambda_1$,
\begin{align*}
2\left(
\frac{|\lambda_1^2| H_{n,2\alpha}}{\tfrac{1}{n}\operatorname{Tr}(\mathbf{H})^2}
- 1
\right)
\left(
\frac{\ln(2n/\delta)}{\varepsilon^2}
\right)
&\le
2\left(
\frac{|\lambda_1^2| H_{n,2\alpha}}{\tfrac{1}{n}|\lambda_1^2|}
- 1
\right)
\left(
\frac{\ln(2n/\delta)}{\varepsilon^2}
\right) \\
&=
2\left(n\,H_{n,2\alpha}-1\right)
\left(
\frac{\ln(2n/\delta)}{\varepsilon^2}
\right) \\
&<
2\left(n\,\zeta(2\alpha)-1\right)
\left(
\frac{\ln(2n/\delta)}{\varepsilon^2}
\right).
\end{align*}
Therefore, $\mathrm{D}^{(s)}$ is a relative $(\varepsilon,\delta)$-approximator if
\(
s
>
2\left(n\,\zeta(2\alpha)-1\right)
\left(
\frac{\ln(2n/\delta)}{\varepsilon^2}
\right).
\)
\end{proof}

\subsection{Non-negativity of the Hessian diagonal of the output layer (\texttt{lm\_head})\label{appendix:proof-prop-lm-head}}

The \texttt{lm\_head} is the only layer whose outputs enter the loss directly. Any change in earlier layers affects the loss through the output produced by the \texttt{lm\_head}. The \texttt{lm\_head} produces logits, which are passed through Softmax to compute the loss. We prove that the \texttt{lm\_head} (output layer) has a non-negative Hessian diagonal (Theorem~\ref{prop:lm_head}). While this non-negativity property is guaranteed for the \texttt{lm\_head}, it is not strictly guaranteed for intermediate layers where non-linear activations (e.g., GeLU) and complex inter-layer dependencies can result in negative Hessian diagonal elements.

\begin{restatable}[Non-negativity of the Hessian diagonal of the output layer]
{theorem}{PropLmHead}
\label{prop:lm_head}
Consider a model whose \emph{output layer} produces the pre-softmax logits
\(
\mathbf{z}(\boldsymbol{\sigma})
=\sum_{i=1}^r \sigma_i\, \mathbf{u}_i \mathbf{v}_i^\top \mathbf{x},
\)
where $\mathbf{x}$ is the layer input, $\{\mathbf{u}_i,\mathbf{v}_i\}_{i=1}^r$ are fixed vectors defining rank-one bases
$\mathbf{u}_i\mathbf{v}_i^\top$, and $\boldsymbol{\sigma}=(\sigma_1,\ldots,\sigma_r)$ are the corresponding output-layer
singular values (scaling coefficients).
Let
\(
\mathbf{p}=\mathrm{softmax}(\mathbf{z}),\,
\ell(\boldsymbol{\sigma})=-\hat{\mathbf{p}}^\top \log \mathbf{p},
\)
where $\hat{\mathbf{p}}$ is a fixed target distribution satisfying $\hat{\mathbf{p}}\succeq \mathbf{0}$ and
$\|\hat{\mathbf{p}}\|_1=1$.
Then, for every $i\in\{1,\dots,r\}$,
\(
\frac{\partial^2 \ell}{\partial \sigma_i^2}\;\ge\;0.
\)
\end{restatable}

\begin{proof}
We proceed by direct differentiation.
For
\[
p_k = \frac{e^{z_k}}{\sum_j e^{z_j}},
\]
Then
\begin{align*}
\frac{\partial p_k}{\partial z_j}
&=
\begin{cases}
\displaystyle
\frac{e^{z_k}}{\sum_{j} e^{z_j}}
-
\frac{e^{2z_k}}{\left(\sum_{j} e^{z_j}\right)^2}
= p_k(1 - p_k),
& k = j, \\[8pt]
\displaystyle
-\frac{e^{z_k} e^{z_j}}{\left(\sum_{j} e^{z_j}\right)^2}
= -p_k p_j,
& k \neq j .
\end{cases}
\end{align*}
Consequently,
\begin{align*}
\frac{\partial \mathbf{p}}{\partial \mathbf{z}}
&=
\mathrm{Diag}(\mathbf{p}) - \mathbf{p}\mathbf{p}^\top .
\end{align*} where $\mathbf{p} = [p_1, \ldots, p_n]^\top$ and $\mathrm{Diag}(\mathbf{p})$ constructs a diagonal matrix with $\mathbf{p}$ as its diagonal. The loss can be written as

\begin{align*}
\ell
&=
-\sum_{k=1}^n \hat p_k \log p_k .
\end{align*}
Its derivative with respect to $p_k$ is
\[
\frac{\partial \ell}{\partial p_k}
=
-\frac{\hat p_k}{p_k}.
\]
Therefore,
\[
\frac{\partial \ell}{\partial \mathbf{p}}
=
-(\mathbf{\hat p} \oslash \mathbf{p})^\top .
\]
The derivative of $\mathbf{z}$ with respect to $\sigma_i$ is
\begin{align*}
\frac{\partial \mathbf{z}}{\partial \sigma_i}
&=
\mathbf{u}_i \mathbf{v}_i^\top \mathbf{x} .
\end{align*}
Applying the chain rule,

\begin{align*}
\frac{\partial \ell}{\partial \sigma_i}
&=
\frac{\partial \ell}{\partial \mathbf{p}}
\,
\left(\frac{\partial \mathbf{p}}{\partial \mathbf{z}}\right)^\top
\,
\frac{\partial \mathbf{z}}{\partial \sigma_i} \\
&=
-(\mathbf{\hat p} \oslash \mathbf{p})^\top
(\mathrm{Diag}(\mathbf{p}) - \mathbf{p}\mathbf{p}^\top)
(\mathbf{u}_i \mathbf{v}_i^\top \mathbf{x}) \\
&=
-\big(\mathbf{\hat p}^\top - \|\mathbf{\hat p}\|_1 \mathbf{p}^\top\big)
(\mathbf{u}_i \mathbf{v}_i^\top \mathbf{x}) .
\end{align*}
As $\|\hat{\mathbf{p}}\|_1 = 1$, this simplifies to
\begin{align*}
\frac{\partial \ell}{\partial \sigma_i}
&=
-(\mathbf{\hat p}^\top - \mathbf{p}^\top)
(\mathbf{u}_i \mathbf{v}_i^\top \mathbf{x}) .
\end{align*}
Differentiating once more,
\begin{align*}
\frac{\partial^2 \ell}{\partial \sigma_i^2}
&=
\frac{\partial}{\partial \mathbf{z}}
\left(
\frac{\partial \ell}{\partial \sigma_i}
\right)
\cdot
\frac{\partial \mathbf{z}}{\partial \sigma_i} \\
&=
\frac{\partial}{\partial\mathbf{p}}
\left(
\frac{\partial \ell}{\partial \sigma_i}
\right)
\cdot
\left(\frac{\partial \mathbf{p}}{\partial \mathbf{z}}\right)^\top
\cdot
\frac{\partial \mathbf{z}}{\partial \sigma_i} \\
&=
(\mathbf{u}_i \mathbf{v}_i^\top \mathbf{x})^\top
(\mathrm{Diag}(\mathbf{p}) - \mathbf{p}\mathbf{p}^\top)
(\mathbf{u}_i \mathbf{v}_i^\top \mathbf{x}) .
\end{align*}
Factoring out the scalar term $(\mathbf{v}_i^\top \mathbf{x})^2$, we obtain
\begin{align*}
\frac{\partial^2 \ell}{\partial \sigma_i^2}
&=
(\mathbf{v}_i^\top \mathbf{x})^2\,
\mathbf{u}_i^\top
(\mathrm{Diag}(\mathbf{p}) - \mathbf{p}\mathbf{p}^\top)
\mathbf{u}_i .
\end{align*}
Next,
\begin{align*}
\mathbf{u}_i^\top
(\mathrm{Diag}(\mathbf{p}) - \mathbf{p}\mathbf{p}^\top)
\mathbf{u}_i
&=
\mathbf{u}_i^\top \mathrm{Diag}(\mathbf{p}) \mathbf{u}_i
-
\mathbf{u}_i^\top \mathbf{p}\mathbf{p}^\top \mathbf{u}_i \\[4pt]
&=
\sum_{k} p_k \, u_{ik}^2
-
(\mathbf{p}^\top \mathbf{u}_i)^2 \\[4pt]
&=
\sum_{k} p_k \, u_{ik}^2
-
2(\mathbf{p}^\top \mathbf{u}_i)^2
+
(\mathbf{p}^\top \mathbf{u}_i)^2 \\[4pt]
&=
\sum_{k} p_k \, u_{ik}^2
-
2 \sum_{k} p_k u_{ik} (\mathbf{p}^\top \mathbf{u}_i)
+
(\mathbf{p}^\top \mathbf{u}_i)^2 .
\end{align*}
Further, because $\sum\limits_{k} p_k = 1$,
\begin{align*}
\mathbf{u}_i^\top
(\mathrm{Diag}(\mathbf{p}) - \mathbf{p}\mathbf{p}^\top)
\mathbf{u}_i
&=
\sum_{k} p_k \, u_{ik}^2
-
2 \sum_{k} p_k u_{ik} (\mathbf{p}^\top \mathbf{u}_i)
+
\sum_{k} p_k (\mathbf{p}^\top \mathbf{u}_i)^2 \\[4pt]
&=
\sum_{k} p_k
\Big[
u_{ik}^2
-
2 u_{ik} (\mathbf{p}^\top \mathbf{u}_i)
+
(\mathbf{p}^\top \mathbf{u}_i)^2
\Big] \\[4pt]
&=
\sum_{k} p_k
\big(
u_{ik} - \mathbf{p}^\top \mathbf{u}_i
\big)^2.
\end{align*}

Because \(p_k \ge 0\) and 
\(\left(u_{ik}-\mathbf{p}^\top\mathbf{u}_i\right)^2 \ge 0\) for every \(k\), it follows that
\[
\mathbf{u}_i^\top
(\mathrm{Diag}(\mathbf{p}) - \mathbf{p}\mathbf{p}^\top)
\mathbf{u}_i
=
\sum_k p_k
\left(u_{ik}-\mathbf{p}^\top\mathbf{u}_i\right)^2
\geq 0.
\]
Therefore,
\begin{align*}
\frac{\partial^2 \ell}{\partial \sigma_i^2}
&=
(\mathbf{v}_i^\top \mathbf{x})^2
\mathbf{u}_i^\top
(\mathrm{Diag}(\mathbf{p}) - \mathbf{p}\mathbf{p}^\top)
\mathbf{u}_i
\;\ge\; 0.
\end{align*}


\end{proof}

\section{Hessian Spectrum in the Singular-Value Space}
\label{sec:sv_hessian_spectrum}
In BSI, each linear weight matrix is reparameterized in terms of its singular-value components. Here, we provide empirical evidence that Assumption~\ref{assump:spectral-power-law} also holds for the Hessian in this reparameterized singular-value space. As an example, we use Llama 2-7B finetuned on a mathematical reasoning dataset~\citep{metamath}. We first reparameterize each weight matrix using SVD. To make the experiment computationally feasible, we fix the smaller singular values and allow only the top 128 singular values of each matrix to vary. This avoids constructing the Hessian with respect to all singular values, which is infeasible for large models due to memory constraints.

To further reduce computational cost, we adopt a layer-wise block-diagonal approximation: we compute Hessian entries only between singular values associated with the same layer and set cross-layer entries to zero. Such block-diagonal curvature approximations are widely used in scalable second-order methods for deep neural networks, where curvature is computed within each layer or block rather than across all parameter pairs~\citep{zhang2018block, martens2015optimizing}. 
We then compute the eigenvalues of the resulting Hessian approximation, sort them by magnitude in descending order, and plot the spectrum on log-log axes. 

Figure~\ref{fig:sv-hessian-spectrum} shows the resulting spectrum. The bulk of the spectrum is consistent with Assumption~\ref{assump:spectral-power-law}: later eigenvalue magnitudes are empirically bounded by a power-law envelope. This is also aligned with prior observations of power-law-type decay in Hessian spectra with respect to the original network weights~\citep{tang2025investigating}. The sharper drop near the tail may be attributed to numerical artifacts, as eigenvalues that are theoretically zero can appear as small nonzero values due to finite-precision errors in the eigendecomposition.

\begin{figure}[H]
\centering
  \centering
  \includegraphics[width=0.55\linewidth]{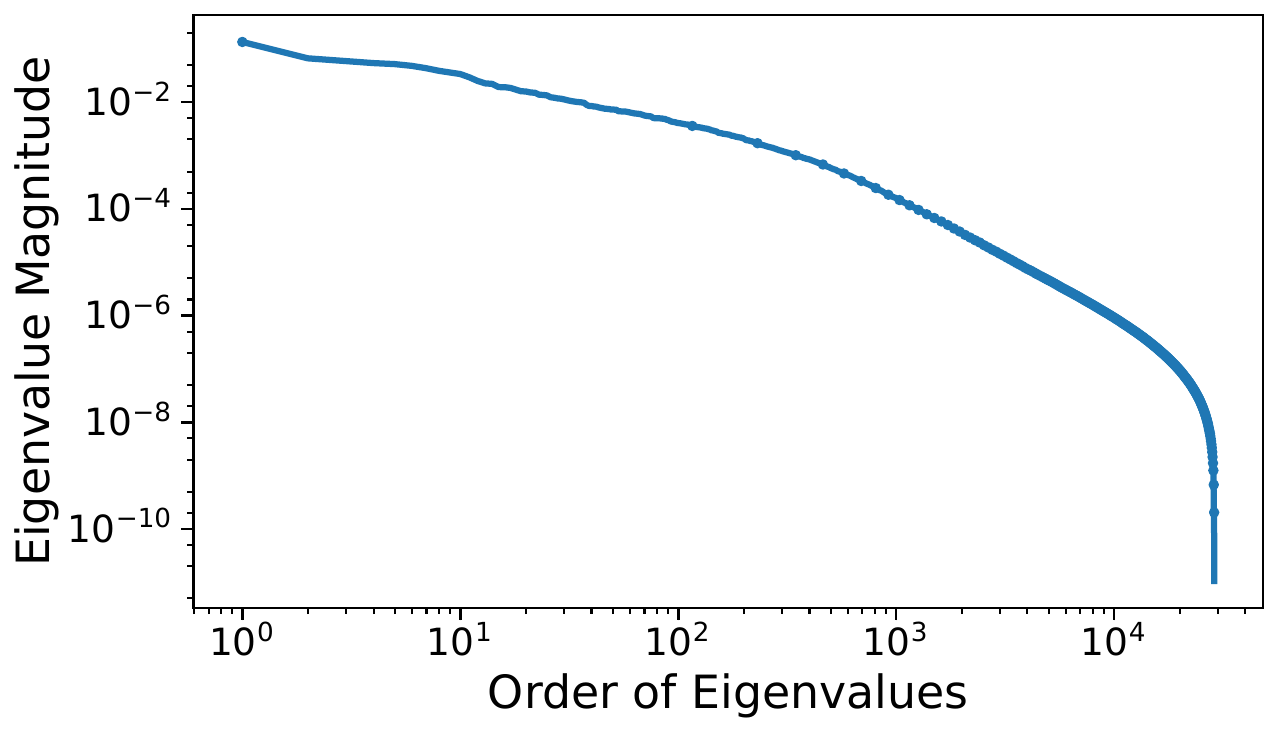}
\caption{Hessian eigenvalue spectrum in the reparameterized singular-value space for math-finetuned Llama 2-7B.}
\label{fig:sv-hessian-spectrum}
\end{figure}


\section{BSI Practical Design Choices
\label{appendx:bsi_implementation_choices}}
This appendix section describes the practical design choices needed for a numerically stable and computationally
efficient implementation of BSI in large-scale models.
Specifically, it addresses (i) the selection of perturbation intensity for
reliable Hessian-diagonal estimation under finite precision, and (ii) the use
of a candidate pool to limit the computational cost of Hessian estimation during pruning.

\subsection{Choosing Perturbation Intensity \label{appendix:perturbation-intensity}} 

The Hessian-diagonal estimator used in the importance score ${I}_i$
relies on finite-difference perturbations of the singular values.
The practical accuracy of this estimator therefore depends critically on the
choice of the perturbation intensity $\epsilon$.
The perturbation magnitude must be chosen
to ensure numerical stability under finite-precision arithmetic.

We consider the floating-point representation of singular value $\sigma_i$,
\begin{align*}
\sigma_i
&= (-1)^{\text{sign}} \cdot (1+\text{fraction}) \cdot 2^{h},
\end{align*}
where the fraction is represented using $f$ bits and $\sigma_i$ denotes the singular value that we perturb when computing the Hessian estimate in Theorem~\ref{thm:hessian_diag_est}.
The rounding error incurs when perturbing $\sigma_i$ by $\epsilon$
satisfies
\begin{equation}
\label{eqn:error_perturbation}
\mathrm{Error}[\sigma_i+\epsilon]
\le \frac{1}{2}\,2^{h-f}.
\end{equation}
As
\(
2^{h}
< |\sigma_i|
< 2^{h+1},
\)
we have
\begin{align*}
\frac{1}{2}|\sigma_i|
&< 2^{h}
< |\sigma_i|.
\end{align*}
Substituting this bound in Eq.~\eqref{eqn:error_perturbation} yields
\begin{align*}
\mathrm{Error}[\sigma_i+\epsilon]
&\le \frac{1}{2}\,2^{-f}\,|\sigma_i|.
\end{align*}
Let $\alpha$ denote the relative error tolerance:
\begin{equation}
\label{eqn:perturb-error}
\frac{\mathrm{Error}(\sigma_i+\epsilon)}{\epsilon}
\le \alpha .
\end{equation}
We need to ensure the largest possible rounding error still satisfies Eq. (\ref{eqn:perturb-error}). 
Therefore,
\[
\frac{\tfrac{1}{2}\,2^{-f}\,|\sigma_i|}{\epsilon}
\le \alpha .
\]
Hence
\begin{align*}
\epsilon
&\ge \frac{1}{2\alpha}\,2^{-f}\,|\sigma_i| \\[4pt]
&= \frac{1}{\alpha}\,2^{-(f+1)}\,|\sigma_i|.
\end{align*}

At the same time, the perturbation intensity must remain sufficiently small to avoid
contamination induced by higher order terms. We therefore bound $\epsilon$ by $\epsilon_{max} < 1$ and choose:
\[
\epsilon
\;=\;
\min\!\left(
\frac{1}{\alpha}\,2^{-(f+1)}\, \sigma_{\max},
\;\epsilon_{max}
\right),
\]
where $\sigma_{\max}$ denotes the largest singular value across all layers of the model and $\epsilon_{\max}$ ensures the perturbation remains small enough for reliable estimation. 
This setting yields a perturbation intensity suitable for all singular values in the model. Using a uniform perturbation intensity reduces estimation variance by alleviating the off-diagonal noise induced by the large singular values.

\subsection{Candidate Pool for Hessian Estimate \label{sec:candidate_pool}} Estimating the Hessian diagonal for all singular values at every pruning step is computationally expensive for large models. To address this, we restrict Hessian estimation to a dynamically constructed \emph{candidate pool} consisting of bases that are likely to be dropped at each pruning round. Formally, let $\mathcal{A}$ denote the set of currently active bases, and let $T$ be the total number of pruning rounds. We first identify the set of most significant bases $\mathcal{K} \subset \mathcal{A}$ as the minimal subset satisfying:
\begin{equation}
\label{eqn:keeping_set}
\sum_{i \in \mathcal{K}} \sigma_i \ge \rho_{\mathrm{cand}} \sum_{j \in \mathcal{A}} \sigma_j,
\qquad
\rho_{\mathrm{cand}} = (\texttt{KeepRatio})^{\gamma / T},
\end{equation}
where  $\texttt{KeepRatio}$ is the final target retention ratio of the model parameters,  and $\gamma>0$ controls the size of the candidate pool. $\rho_{\mathrm{cand}}$ is the  cumulative
sum threshold that we use to form the candidate pool.  At a given pruning round, the \emph{candidate pool} $\mathcal{C}$ is defined as the set of remaining bases:
\begin{equation*}
\label{eqn:candidate_pool}
\mathcal{C} = \mathcal{A} \setminus \mathcal{K}
\end{equation*}
Hessian diagonals are estimated only for singular values corresponding to the bases in the candidate pool. The size of this pool is given by the cardinality $|\mathcal{C}| = |\mathcal{A}| - |\mathcal{K}|$, where $|\mathcal{K}|$ is the number of bases required to satisfy the cumulative sum threshold.

\section{BSI Application \label{appendix:application-BSI}}
\subsection{Evaluation of BSI on Mathematical Reasoning\label{appendix:math-results}}
This section presents results on the performance of models compressed using BSI on mathematical reasoning tasks. We report performance under varying compression ratios. All experiments are performed on NVIDIA A100 GPUs. As Table~\ref{tab:maths-7b} shows, BSI outperforms state-of-the-art low-rank compression baselines at deep compression.

\begin{table}[H]
\centering
\caption{
Accuracy and model size of Llama 2-7B compressed with various low-rank decomposition algorithms on the mathematical reasoning tasks.}
\small

\begin{tabular}{llrrrrrr}
\toprule

\multirow{3}{*}{SVD} 
& \multicolumn{1}{l}{Model Size (Billion)} 
& 6.74 & 5.03 & 3.18 & 1.73 & 1.11 & 0.56   \\\cmidrule{2-8}
& \multicolumn{1}{l}{GSM8K Acc (\%)} 
& 66.41 & 63.00 & 60.96 & 53.90 & 32.90 & 11.90   \\
& \multicolumn{1}{l}{MATH Acc (\%)} 
& 20.56 & 18.30 & 17.38 & 13.74 & 5.34 & 2.76   \\\midrule

\multirow{3}{*}{FWSVD} 
& \multicolumn{1}{l}{Model Size (Billion)} 
& 6.74 & 4.79 & 2.95 & 1.54 & 0.96 & 0.47   \\\cmidrule{2-8}
& \multicolumn{1}{l}{GSM8K Acc (\%)} 
& 66.41 & 62.70 & 62.55 & 56.48 & 1.52 & 1.90   \\
& \multicolumn{1}{l}{MATH Acc (\%)} 
& 20.56 & 19.18 & 17.62 & 14.20 & 1.76 & 1.48   \\\midrule

\multirow{3}{*}{ASVD} 
& \multicolumn{1}{l}{Model Size (Billion)} 
& 6.74 & 4.09 & 3.43 & 2.77 & 1.45 & 0.79  \\\cmidrule{2-8}
& \multicolumn{1}{l}{GSM8K Acc (\%)} 
& 66.41 & 61.94 & 60.88 & 59.74 & 50.49 & 29.11   \\
& \multicolumn{1}{l}{MATH Acc (\%)} 
& 20.56 & 18.74 & 18.72 & 17.30 & 12.62 & 5.18   \\\midrule

\multirow{3}{*}{BASEL} 
& \multicolumn{1}{l}{Model Size (Billion)} 
& 6.74 & 5.11 & 3.27 & 1.82 & 1.20 & 0.41 \\\cmidrule{2-8}
& \multicolumn{1}{l}{GSM8K Acc (\%)} 
& 66.41 & 62.55 & 61.03 & 56.86 & 50.34 & 7.58 \\
& \multicolumn{1}{l}{MATH Acc (\%)} 
& 20.56 & 19.50 & 18.62 & 16.20 & 13.18 & 2.72 \\\midrule

\multirow{3}{*}{BSI} 
& \multicolumn{1}{l}{Model Size (Billion)} 
& 6.74 & 4.52 & 2.85 & 1.59 & 0.51 & 0.34 \\\cmidrule{2-8}
& \multicolumn{1}{l}{GSM8K Acc (\%)} 
& 66.41 & 62.85 & 60.35 & 56.71 & 35.03 & 17.66 \\
& \multicolumn{1}{l}{MATH Acc (\%)} 
& 20.56 & 19.44 & 18.08 & 15.72 & 7.78 & 4.10 \\

\bottomrule
\end{tabular}

\label{tab:maths-7b}
\end{table}



\subsection{Ablation Study}
\label{app:hessian_ablation}

Figures~\ref{fig:BSI_w_o_hessian_7b-math}(a) and~\ref{fig:BSI_w_o_hessian_7b-math}(b) report the ablation results for Llama 2-7B on GSM8K and Hendrycks MATH, respectively. The gradient-only variant (BSI w/o Hessian) uses the same compression pipeline as BSI but removes the Hessian term from the importance metric. The proposed BSI method achieves better performance than the gradient-only variant, indicating that the Hessian term improves the quality of basis selection.

\begin{figure}[H]
\centering
\begin{minipage}[b]{0.48\textwidth}
  \centering
  \includegraphics[width=\linewidth]{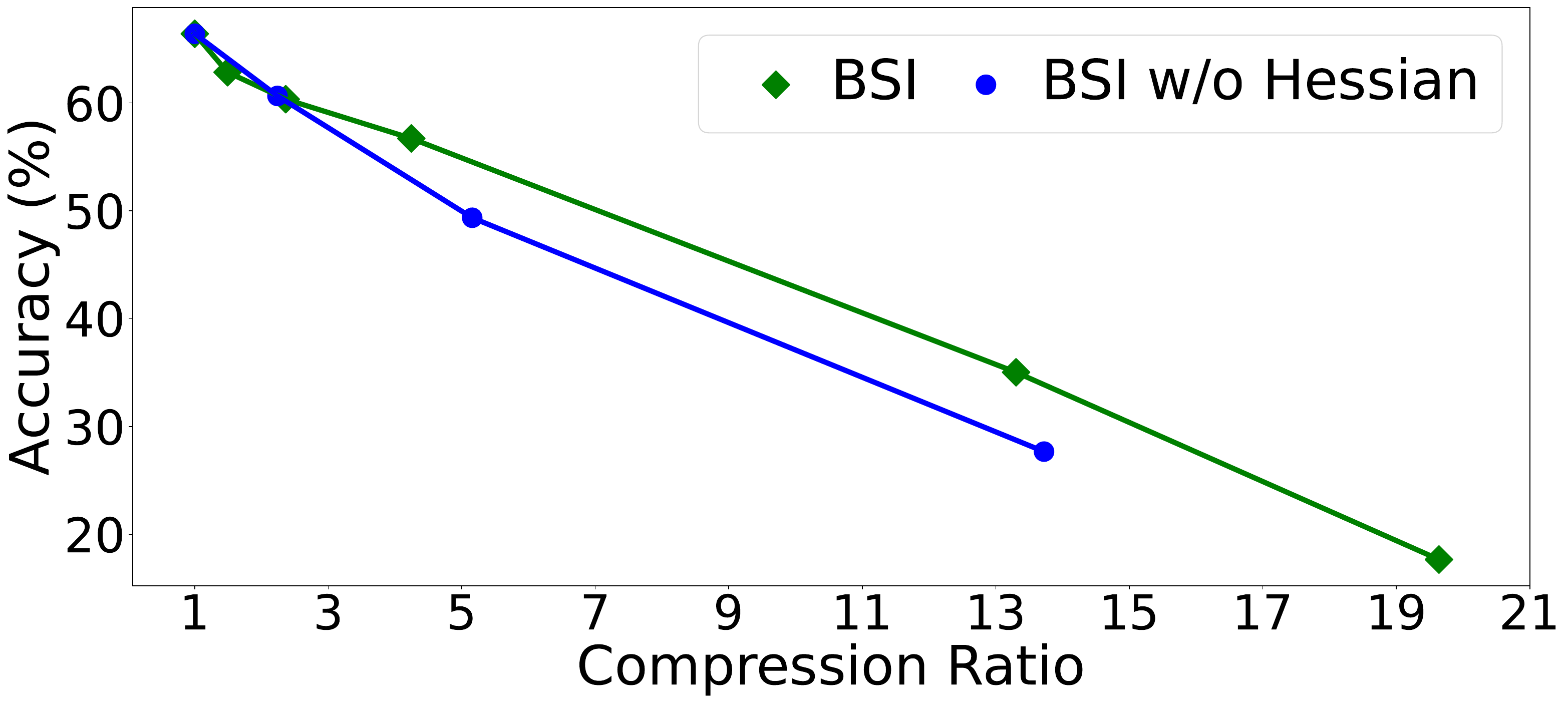}\\
  \small (a) GSM8K
\end{minipage}\hfill
\begin{minipage}[b]{0.48\textwidth}
  \centering
\includegraphics[width=\linewidth]{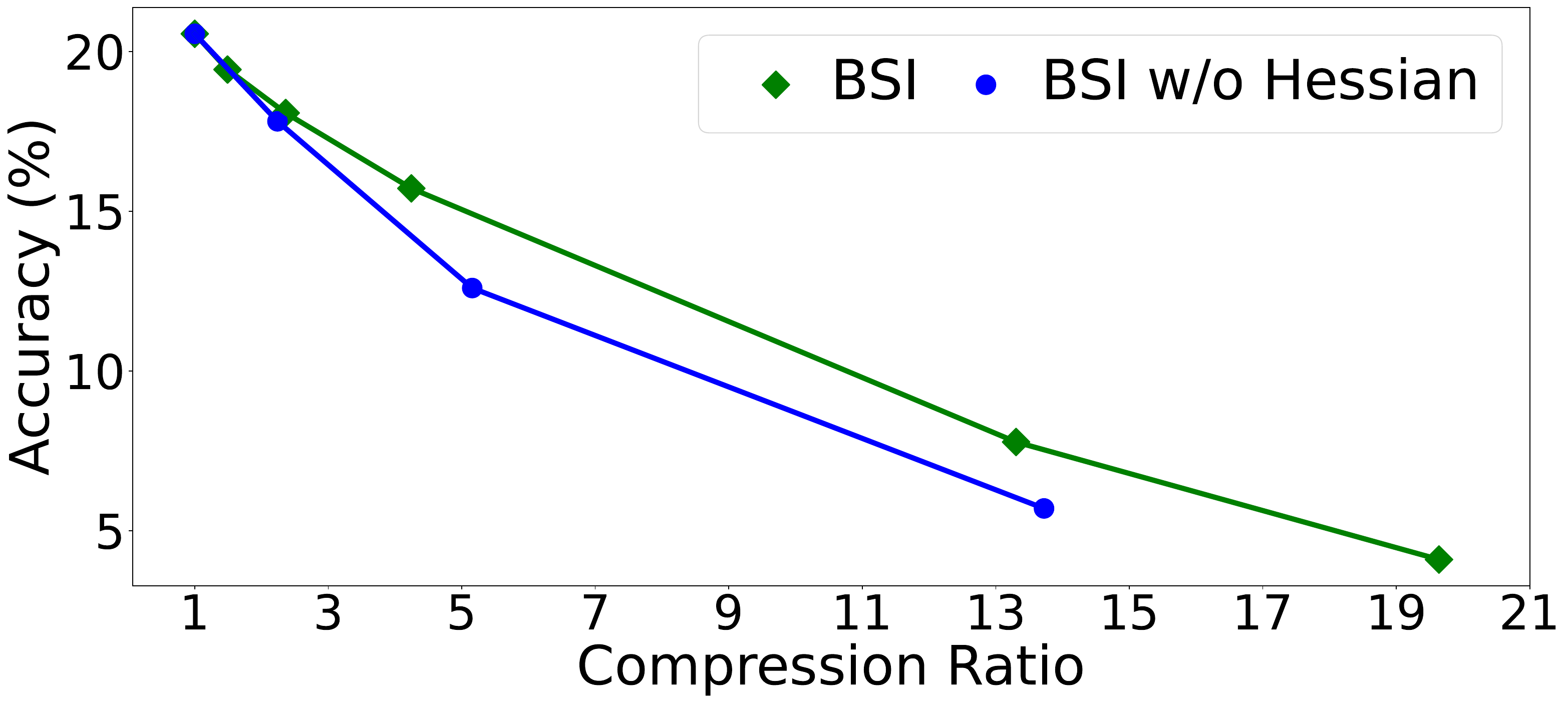}\\
  \small (b) MATH
\end{minipage}
\caption{Accuracy and model size of Llama 2-7B compressed using BSI and its variant without the Hessian term on mathematical reasoning tasks. 
Exact values are shown in Tables~\ref{tab:maths-7b} and \ref{tab:hessian_ablation}.
}
\label{fig:BSI_w_o_hessian_7b-math}
\end{figure}

\begin{table}[H]
\centering
\caption{
Accuracy and model size of Llama 2-7B compressed with BSI without the Hessian term on the mathematical reasoning tasks.}
\small

\begin{tabular}{llrrrr}
\toprule
\multirow{3}{*}{BSI w/o Hessian} 
& \multicolumn{1}{l}{Model Size (Billion)} 
& 6.74 & 3.00 & 1.30 & 0.49 \\
\cmidrule{2-6}
& \multicolumn{1}{l}{GSM8K Acc (\%)} 
& 66.41 & 60.65 & 49.36 & 27.67 \\
& \multicolumn{1}{l}{MATH Acc (\%)} 
& 20.56 & 17.82 & 12.60 & 5.70 \\
\bottomrule
\end{tabular}

\label{tab:hessian_ablation}
\end{table}

\end{document}